\documentclass[11pt]{article}
\usepackage{geometry}
\usepackage{hyperref}
\usepackage{url}

\usepackage{sidecap}
\newcommand{\norm}[1]{\left\lVert#1\right\rVert}
\usepackage{booktabs}         
\usepackage{multirow}         
\usepackage[table]{xcolor}    
\usepackage{array}            
\usepackage{graphicx}   
\usepackage{float}      
\usepackage[utf8]{inputenc} 
\usepackage[T1]{fontenc}    
\usepackage[numbers]{natbib}  
\usepackage{url}            
\usepackage{booktabs}       
\usepackage{amsfonts}       
\usepackage{xspace} 
\usepackage{amsmath}        
\usepackage{cleveref}
\usepackage{nicefrac}       
\usepackage{microtype}      
\usepackage{xcolor}         
\usepackage{wrapfig}
\usepackage{multicol,multirow}
\usepackage{csquotes}
\usepackage{nicematrix}
\usepackage{amsthm}       
\usepackage{thmtools}     
\usepackage{thm-restate}  
\declaretheorem[name=Theorem]{theorem}

\declaretheorem[name=Lemma]{lemma}
\newcommand{\xhdr}[1]{\vspace{1.5mm}\noindent{\bf #1.}\hspace{0.5mm}}
\usepackage{tikz}

\usepackage{mathtools}
\newcommand{\RR}{\mathbb{R}}

\crefname{theorem}{theorem}{theorems}
\crefname{proposition}{proposition}{theorems}
\crefname{lemma}{lemma}{theorems}
\crefname{infoprop}{informal proposition}{theorems}
\usepackage{subcaption}  
\newcommand{\ourmethod}{\texttt{LOS-Net}\xspace}
\newcommand{\ourmethodmlp}{\texttt{ATP+R-MLP}\xspace}
\newcommand{\ourmethodtrans}{\texttt{ATP+R-Transf.}\xspace}
\usepackage{tabularx}
\newcolumntype{Y}{>{\centering\arraybackslash}m{.19\textwidth}} 

\title{Beyond Next Token Probabilities: Learnable, Fast Detection of \\ Hallucinations and Data Contamination on LLM Output Distributions}

\begin{document}

\date{}
\maketitle
\author{}
\begin{center}
\setlength{\tabcolsep}{6pt}      
\renewcommand{\arraystretch}{1.2}

\begin{tabular}{cc}
\textbf{Guy Bar-Shalom}\textsuperscript{*} & \textbf{Fabrizio Frasca}\textsuperscript{*} \\
Technion & Technion \\
{\footnotesize\ttfamily guybs99@gmail.com} &
{\footnotesize\ttfamily fabriziof@campus.technion.ac.il} \\
\end{tabular}

\vspace{1em}

\setlength{\tabcolsep}{4pt}
\renewcommand{\arraystretch}{1.2}

\begin{tabular}{YYYYY}
\textbf{Derek Lim} & \textbf{Yoav Gelberg} & \textbf{Yftah Ziser} & \textbf{Ran El-Yaniv} & \textbf{Gal Chechik} \\
MIT CSAIL & Technion & Nvidia & \shortstack{Technion,\\Nvidia} & \shortstack{Bar-Ilan University,\\Nvidia} \\
\end{tabular}

\vspace{1em}

\begin{tabular}{c}
\textbf{Haggai Maron} \\
\shortstack{Technion,\\Nvidia} \\
\end{tabular}
\end{center}

\begingroup
\renewcommand{\thefootnote}{\fnsymbol{footnote}}
\footnotetext[1]{Equal contribution.}
\endgroup

\begin{abstract}
The automated detection of hallucinations and training data contamination is pivotal to the safe deployment of Large Language Models (LLMs). These tasks are particularly challenging in settings where no access to model internals is available. Current approaches in this setup typically leverage only the probabilities of actual tokens in the text, relying on simple task-specific heuristics. Crucially, they overlook the information contained in the full sequence of next-token probability distributions. We propose to go beyond hand-crafted decision rules by learning directly from the complete observable output of LLMs --- consisting not only of next-token probabilities, but also the full sequence of next-token distributions. We refer to this as the LLM Output Signature (LOS), and treat it as a reference data type for detecting hallucinations and data contamination. To that end, we introduce LOS-Net, a lightweight attention-based architecture trained on an efficient encoding of the LOS, which can provably approximate a broad class of existing techniques for both tasks. Empirically, LOS-Net achieves superior performance across diverse benchmarks and LLMs, while maintaining extremely low detection latency. Furthermore, it demonstrates promising transfer capabilities across datasets and LLMs. Full code is available at \url{https://github.com/BarSGuy/Beyond-next-token-probabilities}.
\end{abstract}

\section{Introduction}

As the remarkable capabilities of LLMs continue to drive their expanding range of applications, detecting hallucinations~\cite{tonmoy2024comprehensive,liu2021token,huang2023survey,ji2023survey,rawte2023troubling}, and training data contamination~\cite{NEURIPS2020_1457c0d6, shi2023detecting,zhang2024min} becomes increasingly important to their reliable deployment and responsible use. Specifically, the tasks of Hallucination and Data Contamination Detection (resp.\ HD, DCD) relate to determining whether an LLM is fabricating information or is providing an incorrect answer to a user question, or whether it has been exposed to specific training data, such as copyrighted material.

\begin{figure*}[!t]
    \centering
    \includegraphics[width=\textwidth]{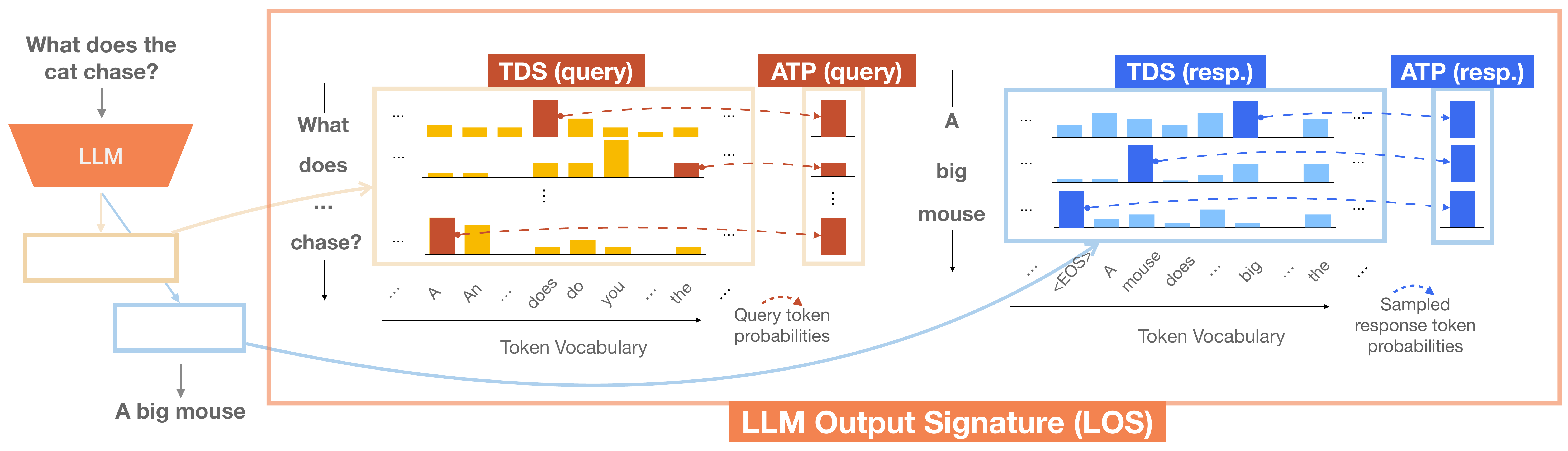}
    \caption{Left: The LLM processes the input ``What does the cat chase?'' and generates the output ``A big mouse''. Right: The corresponding query/response Token Distribution Sequences (TDS) and Actual Token Probabilities (ATP), together constituting the LLM Output Signature (LOS). We propose to detect instances of hallucinations and data contamination by learning directly over this unified data representation, beyond task specific heuristics operating on partial information thereof.}
    \label{fig: TDS}
\end{figure*}

Prominent methods to tackle HD include probing techniques which, although effective, require restrictive white-box access to model internals~\citep{belinkov2022probing,orgad2024llms,hewitt-manning-2019-structural,hewitt2019designing,rateike2023weakly}, such as its hidden states. On both HD and DCD, gray-box methods relax these assumptions by operating \emph{only} on LLM outputs, thus finding application to a broader set of models. These approaches~\citep{guerreiro2022looking,kadavath2022language,varshney2023stitch,huang2023look} typically extract simple features on the sequence of token probabilities, a vector we term Actual Token Probabilities (ATP). However, these methods, often based on heuristics, \emph{overlook} the information contained in the complete next-token probability distributions generated over the token vocabulary at each generation step -- we term this matrix the Token Distribution Sequence (TDS), see \Cref{fig: TDS}. Importantly, this limitation can mask distinctive patterns in the model's text generation process, including its confidence or uncertainty, known to correlate with its correctness~\citep{kuhn2023semanticuncertaintylinguisticinvariances,Farquhar2024}. This aspect is evident even at the level of a single generation step. Consider, e.g., an LLM generating a token with probability $0.5$ in two scenarios: in one case, the remaining next-token probability mass is concentrated on a single alternative $(0.5, 0.5, 0, ..., 0)$, while in the other it is spread across many tokens: $(0.5, 0.01, ..., 0.01)$. See \Cref{fig:TDS imp} in \Cref{app:Importance of TDS Illustration} for an illustration of a similar case. Yet ATP-based approaches would treat them identically. Similarly, an ATP value of $0.1$ at a certain time step could indicate either high uncertainty (if it is the highest probability in a diffused distribution) or strong evidence against the token (if it is a low-ranking probability in a peaked distribution). A recent promising approach \cite{zhang2024min} used some TDS information using heuristics, but a \emph{principled framework} to utilize this data source is still lacking.

\xhdr{Our approach} We argue that a successful gray-box detection approach should leverage both ATP and TDS, together forming what we term the LLM Output Signature (LOS) (\Cref{fig: TDS}) -- the complete observable representation of an LLM in the gray-box setup. Instead of relying on heuristics, we treat LOS as a sequential, high-dimensional and structured data modality on which we apply principled deep learning techniques. We propose \ourmethod, an efficient attention-based model\footnote{Our model features around 1M parameters.} operating on an effective encoding of ATP, TDS, and their interactions. We prove that \ourmethod can approximate a broad class of functions applied to the LOS, subsuming many recent approaches \citep{guerreiro2022looking, kadavath2022language, varshney2023stitch, huang2023look,shi2023detecting,zhang2024min}.
Our comprehensive empirical study on DCD and HD demonstrates a substantial performance gap between using the complete LOS and relying solely on the ATP. Notably, \ourmethod improves over all considered baselines across both tasks, often by a significant margin. Crucially, our architecture is extremely efficient, with detection times of $\sim 10^{-5}$s per instance. This makes it a compelling approach for applications such as on-line error detection for guided-generation, as opposed to previously proposed popular methods based on multiple LLM prompting or generations, such as Semantic Entropy (SE) and P(True)~\citep{kuhn2023semanticuncertaintylinguisticinvariances,kadavath2022language}. \ourmethod also exhibits promising dataset-level transfer and strong cross-model generalization, the latter suggesting its viable application to real-world tasks such as copyright-infringement detection over closed-source LLMs (see, e.g., our results on the BookMIA benchmark~\citep{shi2023detecting} in \Cref{subsec:Data Contamination Detection}). Last, we show \ourmethod retains strong performance also when processing a small subset of the TDS, expressed in terms of the number of top-scoring output probabilities at each generation step. This extends its successful application to LLMs with restricted API access, such as GPT models~\cite{openai2024gpt4technicalreport}.

\xhdr{Contributions} (1) We introduce LOS as a suitable ``data type'' for the detection of hallucinations and data contamination, and develop \ourmethod, an effective and efficient learning framework for that. (2) We show this framework unifies and generalizes previous approaches, and demonstrate it achieves superior performance across models, datasets, and tasks. (3) We find that \ourmethod exhibits strong empirical evidence for cross-model generalization and promising cross-dataset transfer abilities.

\section{Related Work}
We review background and related work on DCD and HD, focusing on studies leveraging logits or output probabilities. Given the breadth of research, we highlight the most relevant works for our setup and refer interested readers to\ref{app:additional background} for further details on these tasks.

\xhdr{Data Contamination Detection (DCD)} This task consists in identifying text passages an LLM has likely seen during training, or memorized. This is crucial for ensuring fair benchmarking of LLMs, guiding dataset curation, and auditing potential copyright infringement. Early methods leveraged model loss \cite{yeom2018privacy,carlini2019secret}, assuming that models overfit their training data. Later refinements introduced \emph{reference models} -- independent LLMs trained on disjoint datasets from a similar distribution -- comparing their scores with the target model~\cite{carlini2021extracting,carlini2022membership}. However, this approach requires access to a well-matched reference model with similar architecture, which is often impractical in real-world settings. Recently, \cite{shi2023detecting} introduced Min-K\%, which flags an input as contaminated if the log probability of its bottom K tokens exceeds a predefined threshold. Building on this approach, \cite{zhang2024min} proposed Min-K\%++, which refines contamination detection by calibrating the next-token log-likelihood using the mean and standard deviation of log-likelihoods across all candidate tokens in the vocabulary. 

\xhdr{Hallucination Detection (HD)} This task has been studied to enable selective intervention, allowing LLMs to prevent fabricated outputs only when necessary \cite{snyder2024early,yin2024characterizing,valentin2024cost}. Recently, \cite{orgad2024llms,azaria2023internalstatellmknows} showed that training a classifier on top of LLMs' hidden states is highly effective for hallucination detection. However, these methods operate under the white-box assumption, requiring full access to the model's internal states. In contrast, our paper explores a more constrained (gray-box) setting, of particular interest especially when targeting closed-source LLMs with restricted API access.

\xhdr{Output probability-Based Analysis} Previous works showed that using log probabilities or raw logits as decision thresholds can be effective for various tasks, including HD in LLMs~\cite{guerreiro2022looking,varshney2023stitch}, correctness self-evaluation \cite{kadavath2022language}, uncertainty estimation \cite{huang2023look}, and zero shot learning \cite{atzmon2019adaptive}. However, these approaches often rely on naive handcrafted thresholding. Other works~\citep{verma-etal-2024-ghostbuster,mosca-etal-2022-suspicious} rely on linear classifiers over features extracted from LLM outputs aiming at tackling adjacent tasks, such as detecting machine-generated text~\citep{wu2023llmdet} but overlook the full TDS, limiting contextual understanding.

\section{Learning on LLM Output Signatures}
In this section, we define the LLM Output Signature (LOS) and introduce \ourmethod, a novel architecture specifically designed to process LOS.

\subsection{LLM Output Signatures (LOS)}
\label{subsec:Notation and Problem Formulation}

Let \( f \) denote a pretrained LLM, and \( \vec{s} \) a text input to \( f \) consisting of $n$ tokens. When queried with $\vec{s}$, $f$ produces outputs $\mathbf{X}_{s} = f(\vec{s})$, i.e., a matrix in $\mathbb{R}^{n \times V}$ of next-token probabilities for each token in $\vec{s}$, where $V$ is the size of the token vocabulary. We define \( \vec{g} \) to be the LLM response to $\vec{s}$, consisting of $m$ tokens generated using $f$'s outputs in $\mathbf{X}_g \in \mathbb{R}^{m \times V}$ (and $\mathbf{X}_s$). We refer to  $\mathbf{X}_s$ or $\mathbf{X}_g$ as \emph{Token Distribution Sequences} (TDS). See \Cref{fig:llm generation}, \Cref{app:LLM Processing Pipeline}. We also define \( \mathbf{p}_s \in \mathbb{R}^n, \mathbf{p}_g \in \mathbb{R}^m \), vectors which holds the probabilities associated with the actual tokens appearing in $\vec{s},\vec{g}$ respectively. We denote these as the \emph{Actual Token Probabilities} (ATP). Specifically, 
$ (\mathbf{p}_s)_i \coloneqq \mathbf{X}_{i,v}$ where $v \in \{1,\dots,V\}$ is the token used in the $i+1$ place in the sequence $\vec{s}$ and similarly for $\vec{g}$ (see \Cref{fig: TDS}). We call the pairs $(\mathbf{X}_s,\mathbf{p}_s)$ or $(\mathbf{X}_g,\mathbf{p}_g)$ the \emph{LLM Output Signature} (LOS).  For DCD, we analyze input sequences using $(\mathbf{X}_s,\mathbf{p}_s)$ as our interest lies in how the model processes the input text $\vec{s}$. For HD, we use $(\mathbf{X}_g,\mathbf{p}_g)$ as we need to make predictions on the model's response. We may use $(\mathbf{X},\mathbf{p})$ if the distinction between the tasks is irrelevant, and use $N$ as the sequence length.

\xhdr{Problem Statement} LOS elements, along with their associated annotations depending on the task of interest, can be gathered into datasets \( D = \{ \big( (\mathbf{X}, \mathbf{p})_i, y_i \big) \}_{i=1}^{\ell} \) where supervised learning problems can be instantiated. Our goal in this paper is to propose a neural architecture that can effectively utilize the complete LOS to solve tasks such as DCD, HD, or any other classification problem thereon.

\subsection{\ourmethod}
Learning from LOS data objects presents inherent challenges, particularly related to their encoding. Next, we detail these challenges and introduce our \ourmethod approach, illustrated in full in \Cref{app:LOS-Net Visualization} and \Cref{fig:our-arch}. What follows is a detailed explanation of each of its components.

\xhdr{Preprocessing the token distribution sequences} Utilizing \(\mathbf{X}\) may pose significant challenges due to three key factors. \emph{(1) Complexity:} The vocabulary tensor can be extremely large in real-world scenarios. For instance, \citet{liang2023xlm} (XLM-V) reported a vocabulary size of 1M tokens, which, for a small batch of documents and popular context sizes, would already entail processing a tensor of tens (or hundreds) of GBs. \emph{(2) Transferability:} Vocabulary size and order may significantly vary between LLMs, something which can complicate transfer learning -- e.g., training on one LLM and applying on another with a different vocabulary size; \emph{(3) Limited Access:} As already mentioned, in certain LLMs, such as those released by OpenAI, the output tensor \(\mathbf{X}\) is only \emph{partially accessible}, with APIs only exposing a small number of the top (log-)probs. To tackle these challenges, we propose selecting, for each row of $\mathbf{X}$, a fixed number of elements. Specifically, we preprocess \(\mathbf{X}\) by sorting each row independently and selecting the top \(K\) probabilities, as follows:
\begin{equation}
\label{eq: X'}
    \mathbf{X'} = \text{row-sort}(\mathbf{X})_{:, :K},
\end{equation}
resulting in \(\mathbf{X'} \in \mathbb{R}^{N \times K}\). 
This approach not only reduces computational complexity but also provides a standardized representation independent of the vocabulary size (for an appropriate choice of $K$). Later, in \Cref{sec:exp}, we will show how even small values of $K$ can capture most of the TDS probability mass and enable strong empirical performance.

\xhdr{Encoding the ATP} The tensor $\mathbf{X}'$ provides a comprehensive description of the LLM's output, but does not explicitly encode an important source of information: the probability $\mathbf{p}$ of the actual tokens appearing in the sequence, i.e, the ATP. 
While these values are technically present in the TDS (since it contains the full distribution), they are not directly distinguishable from the other token probabilities in the vocabulary. Thus, we do also include ATPs as separate inputs to our architecture and further complement these probabilities with additional information which allows us to contextualize them with respect to the whole TDS. Specifically, we argue that valuable information is encoded in the \emph{rank}\footnote{The rank of the $i$-th token is defined as: $r_i(\mathbf{X}, \mathbf{p}) = \sum_{v=1}^V \mathbb{I}(\mathbf{X}_{i, v} > p_i),$ where $\mathbb{I}(\cdot)$ is the indicator function.} (position) of the ATP within the sorted TDS. This information reveals the ``gap'' between the actual token and the token the model would most likely expect to find instead. We encode the rank in a way to make this feature more amenable for learning: we apply scaling to a closed interval and transform it with specific parameters, obtaining $\text{RE}(\mathbf{X}, \mathbf{p})$. More details are found in Appendix~\ref{app:our-baselines}. 

\xhdr{Architecture} Given the preprocessed TDS \(\mathbf{X}'\) and the rank encoding $\text{RE}(\mathbf{X}, \mathbf{p})$, we first linearly project \(\mathbf{X}'\) via $\mathbf{W} \in \mathbb{R}^{K \times K'}$, concatenate it with $\text{RE}(\mathbf{X}, \mathbf{p})$, and then feed it to an encoder-only transformer module \(\mathcal{T}\) with learnable positional encodings, operating in the sequence dimension \citep{vaswani2017attention}:
\begin{align}
\label{eq: arch}
    h_\theta(\mathbf{X}, \mathbf{p}) = \mathcal{T} \left( \mathbf{X}'\mathbf{W} \, \bigg\| \, \text{RE}(\mathbf{X}, \mathbf{p})  \right).
\end{align}
\noindent Here, $\theta$ includes all model's parameters, $\|$ denotes concatenation on the feature dimension. Finally, we pool over the [CLS] token and obtain output scores via a linear layer. The resulting model, \ourmethod, is trained with binary cross-entropy loss.

\section{Generalization of Previous Approaches}
\label{subsec:TDSnet provably generalizes previous approaches}

As already mentioned, prior research has introduced various gray-box, methods for HD and DCD based on LLM's output probabilities \cite{guerreiro2022looking,kadavath2022language,varshney2023stitch,huang2023look}. In what follows, we propose a general framework to unify these diverse techniques, and show that this can be captured by our \ourmethod method, shedding light on its flexibility.

\xhdr{Motivating example: Min-K\% ~\citet{shi2023detecting}} Min-K\%, a prominent, recent method for DCD, makes predictions on an input text $\vec{s}$ based on a score $R$ calculated as the average of the smallest $K$\% log-probs:  $R(\vec{s}) = \frac{1}{|M|} \sum_{i \in M} \log (p_i)$, with $M = \{i \mid p_i < \text{perc}(\mathbf{p}, K) \}$ being the set of token indices whose probabilities are in the first $K$-th percentile of $\mathbf{p}$. We note that it is instructive to rewrite the scoring equation as: %
\begin{equation}
    R(\vec{s}) = \sum_{i=1}^{|\vec{s}|} \overbrace{\frac{\log (p_i)}{\left\lceil \frac{K}{100} \cdot |\vec{s}| \right\rceil}}^{\text{token-wise score}} \cdot \underbrace{\mathbb{I}\big(\overbrace{p_i}^{\text{confidence}} < \overbrace{\text{perc}(\mathbf{p}, K)}^{\text{adaptive threshold}}\big)}_\text{gating}.
\end{equation}
This highlights a general pattern: that of computing a global score by aggregating token-wise values meeting a (dynamic) ``acceptance'' condition, a form of ``gating''. To unify the aforementioned baselines under a common framework, we formalize this pattern via a family of functions (see next).

\xhdr{Gated Scoring Functions (GSFs)} We define the family of \emph{Gated Scoring Functions} (GSF) as the set of functions scoring LOSs by aggregating token-wise scores across the input sequence whenever their confidence values exceed a (possibly adaptive) threshold. GSFs are described in terms of the following components:
(1) A confidence function $\kappa: \mathbb{R}^{N\times k} \times \mathbb{R}^{N} \rightarrow \mathbb{R}^N$ that assigns confidence values to each token in the sequence; (2) A threshold function $T: \mathbb{R}^{N\times k} \times \mathbb{R}^{N} \rightarrow \mathbb{R}$ that determines an acceptance criterion; and (3) A weight function $g: \mathbb{R}^{N\times k} \times \mathbb{R}^{N} \rightarrow \mathbb{R}^N$ that assigns importance scores to tokens. Given a LOS $(\mathbf{X}, \mathbf{p})$, a GSF computes a global score $R(\mathbf{X}, \mathbf{p})$ as follows:
\begin{align}
    F(\mathbf{X}, \mathbf{p})_i &=
    \begin{cases}
    g(\mathbf{X}', \mathbf{p})_i, & \text{if } \kappa(\mathbf{X}', \mathbf{p})_i \geq T(\mathbf{X}', \mathbf{p}), \nonumber \\
    0, & \text{otherwise},
    \end{cases}, \\
    R(\mathbf{X}, \mathbf{p}) &= \sum_{i=1}^N F(\mathbf{X}, \mathbf{p})_i, \label{eq:GSF}
\end{align}
where $\mathbf X'$ is the sorted version of $\mathbf X$, as per \Cref{eq: X'}. 
The family of GSF is flexible enough to capture previously proposed gray-box methods, as we show in the following:
\begin{restatable}[GSFs capture known baselines]{proposition}{GSFsCaptureBaselines}\label{prop:GSFs_capture_baselines}
    Let $\mathcal{B}$ be the set of scoring functions implemented by the Min/Max/Mean aggregated probability methods~\citep{guerreiro2022looking, kadavath2022language, varshney2023stitch, huang2023look} for HD, as well as Loss~\citep{yeom2018privacy}, the MinK\%~\citep{shi2023detecting} and MinK\%++~\citep{zhang2024min} methods for DCD. For any scoring function $f \in \mathcal{B}$, there exists a choice of functions $\kappa, T, g$ such that the GSF $R$ in~\Cref{eq:GSF}, implements $f$.
\end{restatable}
\noindent It is easy to see, e.g., how MinK\% is implemented as a GSF\footnote{For a sequence length of $N$, it suffices to choose: $T(\mathbf{X}', \mathbf{p}) = -\text{perc}(\mathbf{p}, K) = -\big(\text{sort}(\mathbf{p})_{\left\lceil \frac{K}{100} \cdot N \right\rceil} \big), \quad \kappa(\mathbf{X}', \mathbf{p}) = -\mathbf{p}, \quad g(\mathbf{X}', \mathbf{p}) = \frac{\log \mathbf{p}}{\left\lceil \frac{K}{100} \cdot N \right\rceil}.$}. Refer to \Cref{app:Proofs} for more details on how other baselines are implemented.

\xhdr{\ourmethod can approximate GSFs and implement known baselines} As the following result shows, our \ourmethod  can, in fact, approximate virtually all GSFs of interest; intuitively, there exist sets of parameters such that it evaluates ``arbitrarily close'' to the target GSFs.

\begin{restatable}[\ourmethod can approximate \Cref{eq:GSF}]{proposition}{ThresholdImplementation}\label{prop:ThresholdImplementation}
Assume maximal possible vocabulary size \(V_\text{max}\) and context size \(N_{\text{max}}\). Let \(\mathcal{X} \times \mathcal{M} \subseteq \mathbb{R}^{N_{\text{max}} \times V_\text{max}} \times \mathbb{R}^{N_{\text{max}}}\) represent a compact subset in the LOS. For any measurable $\kappa : \mathcal{X} \times \mathcal{M} \rightarrow \mathbb{R}^{N_{\text{max}}}$, measurable $T: \mathcal{X} \times \mathcal{M} \rightarrow \mathbb{R} $, measurable and integrable weight function $g: \mathcal{X} \times \mathcal{M} \rightarrow \mathbb{R}^{N_{\text{max}}}$, and for any $\epsilon > 0$, there exists a set of parameters $\theta$ such that our model $h_\theta : \mathcal{X} \times \mathcal{M} \rightarrow \mathbb{R}$ satisfies $\norm{h_\theta - R}_{L_1} < \epsilon$ where $\norm{\cdot}_{L_1}$ denotes the $L_1$ norm.
\end{restatable}
To prove this result, we build on existing universality results on approximating continuous functions with Transformers~\citep{yun2019transformers}, showing that our (generally non-continuous) target functions can be approximated by continuous functions. Importantly, \Cref{prop:ThresholdImplementation} implies that, as long as the LOS space of interest lies within a compact domain\footnote{This is inherently satisfied when using probabilities; or via clamping in the case of logits or log-probs.}, our model can approximate the general GSF in \Cref{eq:GSF} of LOSs for any LLM under mild conditions on $\kappa$, $T$, and $g$. Note that \Cref{prop:ThresholdImplementation} cannot be generally extended to $L_{\infty}$ due to the discontinuity of GSFs. The practical relevance of~\Cref{prop:ThresholdImplementation}, is underscored by the following:
\begin{restatable}[Approximation of Baselines by \ourmethod]{corollary}{BaselinesApprox}\label{prop:BaselinesApprox}
Our architecture, as defined in ~\Cref{eq: arch}, can arbitrarily well approximate, in the \(L_1\) sense, any of the baseline methods in $\mathcal{B}$ when operating on context and token-vocabulary of, resp., maximal sizes $N_{\text{max}}$ and $ V_{\text{max}}$.
\end{restatable}
The above states that well-established, successful baselines (see class $\mathcal{B}$ in ~\Cref{prop:GSFs_capture_baselines}) can be approximated by \ourmethod. All proofs are enclosed in \Cref{app:Proofs}.

\begin{figure*}[t] 
\centering

\begin{minipage}[t]{0.60\textwidth}
  \begin{table}[H]
\resizebox{\columnwidth}{!}{
\begin{tabular}{l|ccc|ccc}
    \toprule
    \multirow{2}{*}{Method} & HotpotQA & IMDB & Movies & HotpotQA & IMDB & Movies \\
     \cmidrule{2-7}
    & \multicolumn{3}{c|}{Mistral-7b-instruct} & \multicolumn{3}{c}{Llama3-8b-instruct} \\
    \midrule
    Logits-mean          & 
        61.00 \scalebox{0.8}{$\pm$ 0.20}  & 
        57.00 \scalebox{0.8}{$\pm$ 0.60} & 
        63.00 \scalebox{0.8}{$\pm$ 0.50} & 
        65.00 \scalebox{0.8}{$\pm$ 0.20} & 
        59.00 \scalebox{0.8}{$\pm$ 1.70} & 
        75.00  \scalebox{0.8}{$\pm$ 0.50} \\
    Logits-min           & 
        61.00 \scalebox{0.8}{$\pm$ 0.30}  & 
        52.00 \scalebox{0.8}{$\pm$ 0.70} & 
        66.00 \scalebox{0.8}{$\pm$ 0.80} & 
        67.00 \scalebox{0.8}{$\pm$ 0.80} & 
        55.00 \scalebox{0.8}{$\pm$ 1.60} & 
        71.00 \scalebox{0.8}{$\pm$ 0.50} \\
    Logits-max           & 
        53.00 \scalebox{0.8}{$\pm$ 0.80}  & 
        47.00 \scalebox{0.8}{$\pm$ 0.40} & 
        54.00 \scalebox{0.8}{$\pm$ 0.40} & 
        59.00 \scalebox{0.8}{$\pm$ 0.50} & 
        51.00 \scalebox{0.8}{$\pm$ 0.90} & 
        67.00 \scalebox{0.8}{$\pm$ 0.30} \\
    Probas-mean          & 
        63.00 \scalebox{0.8}{$\pm$ 0.30} & 
        54.00 \scalebox{0.8}{$\pm$ 0.80} & 
        61.00 \scalebox{0.8}{$\pm$ 0.20} & 
        61.00 \scalebox{0.8}{$\pm$ 0.20} & 
        73.00 \scalebox{0.8}{$\pm$ 1.50} & 
        73.00  \scalebox{0.8}{$\pm$ 0.60} \\
    Probas-min           & 
        58.00 \scalebox{0.8}{$\pm$ 0.30} & 
        51.00 \scalebox{0.8}{$\pm$ 1.00} & 
        60.00 \scalebox{0.8}{$\pm$ 0.80} & 
        60.00 \scalebox{0.8}{$\pm$ 0.40} & 
        57.00 \scalebox{0.8}{$\pm$ 1.60} & 
        65.00 \scalebox{0.8}{$\pm$ 0.40} \\
    Probas-max           & 
        50.00 \scalebox{0.8}{$\pm$ 0.50} & 
        48.00 \scalebox{0.8}{$\pm$ 0.40} & 
        51.00 \scalebox{0.8}{$\pm$ 0.50} & 
        56.00 \scalebox{0.8}{$\pm$ 0.50} & 
        49.00 \scalebox{0.8}{$\pm$ 0.80} & 
        64.00 \scalebox{0.8}{$\pm$ 0.60} \\
    \midrule
    \rowcolor{orange!10} P(True)              & 
        54.00 \scalebox{0.8}{$\pm$ 0.60} & 
        62.00 \scalebox{0.8}{$\pm$ 0.90} & 
        62.00 \scalebox{0.8}{$\pm$ 0.50} & 
        55.00 \scalebox{0.8}{$\pm$ 0.50} & 
        60.00 \scalebox{0.8}{$\pm$ 0.60} & 
        66.00 \scalebox{0.8}{$\pm$ 0.40} \\
    \rowcolor{orange!10} Semantic Entropy & 
        67.66 \scalebox{0.8}{$\pm$ 0.55} & 
        62.44 \scalebox{0.8}{$\pm$ 0.81}  & 
        \underline{70.24} \scalebox{0.8}{$\pm$ 0.68} & 
        65.58 \scalebox{0.8}{$\pm$ 0.53} & 
        74.96 \scalebox{0.8}{$\pm$ 1.00} & 
        72.27 \scalebox{0.8}{$\pm$ 0.65} \\
    \midrule
    \ourmethodmlp           & 
        68.92 \scalebox{0.8}{$\pm$ 0.24}  & 
        90.70 \scalebox{0.8}{$\pm$ 0.50}  & 
        66.04 \scalebox{0.8}{$\pm$ 0.13}  & 
        64.50 \scalebox{0.8}{$\pm$ 0.75}  & 
        \underline{88.68} \scalebox{0.8}{$\pm$ 0.30}  & 
        73.25 \scalebox{0.8}{$\pm$ 0.15}  \\
    \ourmethodtrans       & 
        \underline{69.70} \scalebox{0.8}{$\pm$ 0.39} & 
        \underline{89.64} \scalebox{0.8}{$\pm$ 1.08} & 
        67.92 \scalebox{0.8}{$\pm$ 0.98}  & 
        \underline{66.72} \scalebox{0.8}{$\pm$ 0.39}  & 
        \underline{85.46} \scalebox{0.8}{$\pm$ 1.14} & 
        \underline{75.89} \scalebox{0.8}{$\pm$ 1.07}   \\
    \textbf{\ourmethod}      & 
        \textbf{72.92} \scalebox{0.8}{$\pm$ 0.45}  & 
        \textbf{94.73} \scalebox{0.8}{$\pm$ 0.58} & 
        \textbf{72.20} \scalebox{0.8}{$\pm$ 0.66}  & 
        \textbf{72.60} \scalebox{0.8}{$\pm$ 0.34} &
        \textbf{90.57} \scalebox{0.8}{$\pm$ 0.28}  & 
        \textbf{77.43} \scalebox{0.8}{$\pm$ 0.66}  \\  
        \midrule
        \midrule
        \rowcolor{red!10} Act. Probe (incomp.$^\dagger$) & 
        73.00 \scalebox{0.8}{$\pm$ 0.60} & 
        92.00 \scalebox{0.8}{$\pm$ 1.00} & 
        72.00 \scalebox{0.8}{$\pm$ 0.50} & 
        77.00 \scalebox{0.8}{$\pm$ 0.50} & 
        81.00 \scalebox{0.8}{$\pm$ 1.40} & 
        78.00 \scalebox{0.8}{$\pm$ 0.40} \\
    \bottomrule
\end{tabular}
}
    \caption{Test AUCs for HD over Mis-7b and L-3-8b (\textbf{bold}: best method, \underline{underlined}: second best). \colorbox{orange!10}{orange} indicates baselines requiring additional prompting/generations. $^\mathbf{\dagger}$\colorbox{red!10}{Activation Probes}, included as reference, are \emph{incomparable as they access model internals}. 
    }
    \label{tab:hall-merged}
  \end{table}
\end{minipage}%
\hfill
\begin{minipage}[t]{0.38\textwidth}
    \begin{figure}[H]
        \centering
        \includegraphics[width=0.6\textwidth]{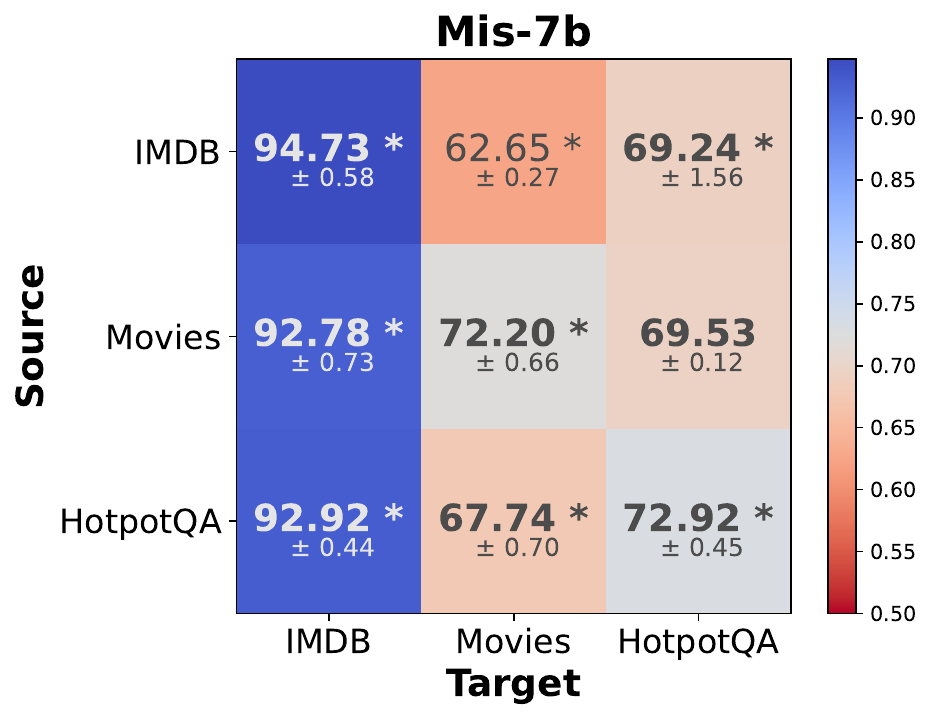}
        \includegraphics[width=0.6\textwidth]{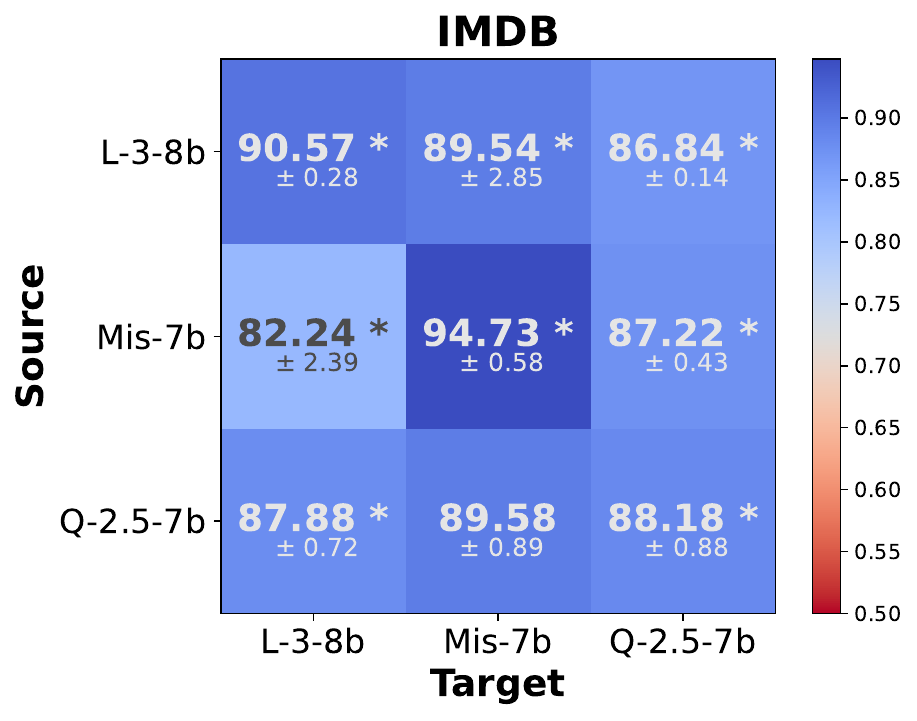}
    \caption{Transfer Test AUC to varying datasets (top, Mis-7b) and LLMs (bottom, IMDB fixed).}
    \label{fig:transfer_hd_main}

    \end{figure}
\end{minipage}
\end{figure*}

\section{Experiments}\label{sec:exp}

We assess various aspects of learning with LOS via the following questions: (1) Is learning on LOS an effective approach for DCD and HD? Does it outperform baselines? And how important is $\mathbf{X}$, i.e., the TDS, in the pair $(\mathbf{X}, \mathbf{p})$, often overlooked? (\Cref{subsec:Data Contamination Detection,subsec:Hallucination Detection}); (2)  Does our model exhibit transfer capabilities across LLMs and datasets? (\Cref{subsec:Generalization to other LLMs and datasets}); (3) What is the practical runtime of our approach and how robust is it w.r.t.\ $K$ (\Cref{subsec:Run-Time,subsec:k-ablation})? 

In the following, we present our main results, and refer to \Cref{app:exp} for additional results and details.

\xhdr{General setup} Our experiments focus on the two tasks of DCD and HD; in all results presented next, hyperparameter $K$ is always set to $1,000$, while its impact is discussed in~\Cref{subsec:k-ablation}. Aligning with prior work, we use datasets and LLMs from \cite{shi2023detecting,zhang2024min} for DCD and \cite{orgad2024llms} for HD, where we also experiment with an additional LLM (Qwen-2.5-7b-Instruct~\citep{qwen2}). Further details are in subsequent sections. As performance metric, we use the area under the ROC curve (AUC), a standard metric in this domain \cite{orgad2024llms,shi2023detecting,zhang2024min}. We run each experiment with three different random seeds (when applicable) and report the mean along with the standard deviation of the results. All \ourmethod experiments were conducted using PyTorch \cite{paszke2019pytorch} on a single NVIDIA L-40 GPU.

\xhdr{Newly introduced learning-based baselines} In addition to task-specific baselines, we also introduce two novel learning-based baselines to appreciate the contribution of the TDS: \ourmethodmlp, \ourmethodtrans. Specifically, we ablate information about the TDS and only process the ATP and rank information with, resp., an MLP or Transformer backbone. Formal definitions are in \Cref{app:our-baselines}.

\subsection{Hallucination Detection}
\label{subsec:Hallucination Detection} 

\xhdr{Datasets and LLMs} We adopt datasets from \citet{orgad2024llms}, following the same setup: the objective is to predict whether an LLM-generated response to a given input prompt is correct or not. We choose three datasets spanning various domains and tasks: HotpotQA without context \cite{yang2018hotpotqa}, IMDB sentiment analysis \cite{maas2011learning}, roles in Movies \cite{orgad2024llms}. Details regarding the annotation process, splits and dataset sizes are in \Cref{app:Datasets for Hallucination Detection}. As the target LLMs, coherently with \citet{orgad2024llms}, we use Mistral-7b-instruct-v0.2 \cite{jiang2023mistral} (Mis-7b) and LlaMa3-8b-instruct \cite{touvron2023llama} (L-3-8b), and further experiment with Qwen-2.5-7b-Instruct~\citep{qwen2} (Q-2.5-7b).

\xhdr{HD Baselines} 
\begin{enumerate}
    \item Aggregated probabilities/logits~\cite{guerreiro2022looking,kadavath2022language,varshney2023stitch,huang2023look}. They simply operate mean/max/min pooling over the ATP to score LLM confidence for error detection. We refer to them as Logit/Probas-mean/min/max.
    \item \colorbox{orange!10}{P(True)} -- \citet{kadavath2022language} found that LLMs show reasonable calibration in assessing their own correctness via additional querying.
    \item \colorbox{orange!10}{Semantic Entropy}~\citep{Farquhar2024,kuhn2023semanticuncertaintylinguisticinvariances}: a popular technique resorting to additional generations and an auxiliary entailment model to assess the uncertainty of LLM's responses at a semantic level, argued, in turn, to be predictive for correctness. Note that both this method and the above P(True) require additional prompting and/or generations, making their detection latency orders of magnitudes higher than other methods in comparison, refer to~\Cref{subsec:Run-Time}.
    \item \colorbox{red!10}{Activation Probes}~\citep{orgad2024llms,azaria2023internalstatellmknows,belinkov2022probing} are linear classifiers fitted over the LLM's internal activations. They operate in the more restrictive white-box setup, thus \emph{not directly comparable} to \ourmethod. Still, they constitute relevant performance references. We probe the last generated token at the layer maximizing validation performance. 
\end{enumerate}

\xhdr{Results} \Cref{tab:hall-merged} presents a comprehensive summary of results on Mis-7b and L-3-8b. These clearly demonstrate that \ourmethod outperforms all gray-box baselines across all six dataset/LLM combinations, often by a significant margin. These also include P(True) and Semantic Entropy, which use auxiliary prompts or generations. We highlight how, on the IMDB dataset, \ourmethod achieves an AUC improvement of around 31 units over the best of these baselines for Mis-7b and 17 over the best baseline for L-3-8b. Intriguingly, we note how \ourmethod outperforms even white-box Activation Probes in 2 out of 6 combinations, while performing similarly in 3 of them. Our results also indicate that ATP learning-based baselines consistently underperform compared to \ourmethod, underscoring the critical role of the TDS, $\mathbf{X}$. Our ATP-based learnable baselines still outperform non-learnable probability-based methods in most cases, suggesting that a learning approach relying exclusively on ATP can still be a viable solution in certain scenarios. Results on Q-2.5-7b are consistent with the above findings, and are deferred to \Cref{app:Results On Qwen}.

\subsection{Data Contamination Detection}\label{subsec:Data Contamination Detection}

DCD is often framed as a Membership Inference Attack (MIA)~\citep{shokri2017membership,mattern2023membership,shi2023detecting}. A DC dataset $D = \{q_i, y_i\}_{i=1}^{\ell}$ contains ${\ell}$ text samples, where $q_i$ represents the text and $y_i$, the target, indicates whether it was part of the training data or not.

\xhdr{Datasets and LLMs} We use three datasets to assess DCD, specifically: WikiMIA-32 and WikiMIA-64~\cite{shi2023detecting} (excerpts from Wikipedia articles), as well as BookMIA~\cite{shi2023detecting} (excerpts from books). Henceforth, due to space limitations, we will only discuss details and results related to the latter, while referring readers to \Cref{app:Results On The WikiMIA Dataset} for the former. In BookMIA, positive members correspond to books known to be well memorized by certain OpenAI models~\citep{chang2023speak}, or otherwise known to (partly) be in the pretraining corpus of other open-source LLMs~\citep{antebi2025tag}. Non-members include excerpts from books released after 2023, necessarily absent from the pretraining corpus of the last ones. This dataset allows us to test \ourmethod in a realistic scenario akin to copyright-infringement detection. In particular, contrary to previous works, we propose a novel split that ensures all excerpts from the same book always appear either in the training or test split (and never in both). Details are enclosed in \Cref{app:bookmia}. We attack LLMs considered in~\citep{antebi2025tag}: LlaMa-13b/30b \cite{touvron2023llama} (L-13b/30b), Pythia-6.9b/12b\cite{biderman2023pythia} (P-6.9b/12b).

\begin{figure*}[t] 
\centering

\begin{minipage}[t]{0.55\textwidth}
  \begin{table}[H]
\resizebox{\columnwidth}{!}{
\begin{tabular}{l|cccc}
    \toprule
    Method / LLM & P-6.9b & P-12b & L-13b & L-30b \\
    \midrule
    Loss                  & 67.40  & 76.27  & 76.23  & 89.18  \\
    MinK            & 68.78  & 77.32  & 75.36  & 89.61  \\
    MinK++          & 66.73  & 71.76  & 72.87  & 80.60  \\
    \midrule
    \rowcolor{red!10}  Zlib            & 50.01 & 60.84 & 61.94 & 80.83 \\
    \rowcolor{red!10}  Lowercase           & 74.97  & 81.64  & 67.80  & 82.18  \\
    \rowcolor{red!10}  Ref             & \underline{89.52}  & \textbf{91.93}  & \underline{84.58}  & \underline{94.93}  \\
    \midrule
    \ourmethodmlp           & 56.31 \scalebox{0.8}{$\pm$ 1.48}  & 57.18 \scalebox{0.8}{$\pm$ 1.06}  & 66.60 \scalebox{0.8}{$\pm$ 1.05} & 83.89 \scalebox{0.8}{$\pm$ 0.41}  \\
    \ourmethodtrans       & 79.59 \scalebox{0.8}{$\pm$ 0.61}  & 74.77 \scalebox{0.8}{$\pm$ 0.57} & 74.65 \scalebox{0.8}{$\pm$ 0.79}  & 87.62 \scalebox{0.8}{$\pm$ 0.68}  \\
    \textbf{\ourmethod}      & \textbf{90.71} \scalebox{0.8}{$\pm$ 0.90}  & \underline{89.43} \scalebox{0.8}{$\pm$ 0.59} & \textbf{91.02} \scalebox{0.8}{$\pm$ 0.15} & \textbf{95.60} \scalebox{0.8}{$\pm$ 0.41}  \\        \bottomrule
\end{tabular}
}

    \caption{Test AUCs on BookMIA. `P': Pythia, `L': LlaMa-1 (\textbf{bold}: best, \underline{underlined}: second best, \colorbox{red!10}{pink}: reference-based).}
    \label{tab:bookmia-results}
  \end{table}
\end{minipage}%
\hfill
\begin{minipage}[t]{0.40\textwidth}
    \centering
    \begin{figure}[H]
        \centering
        \includegraphics[width=0.7\columnwidth]{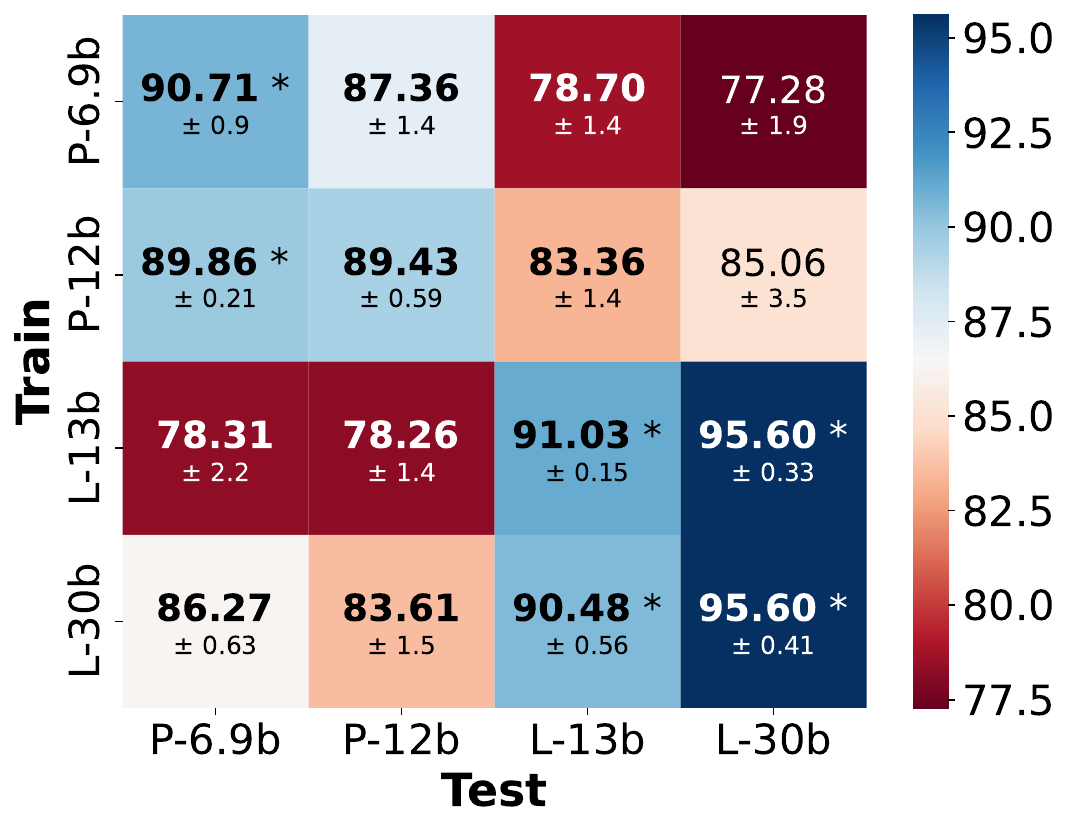}
        \caption{BookMIA zero-shot AUC -- \textbf{bold}, $^*$: outperforms, resp., ref-free and ref-based.}
        \label{fig:gene-bookmia}
    \end{figure}
\end{minipage}

\end{figure*}

\xhdr{DCD Baselines} The Loss approach~\cite{yeom2018privacy} directly uses the loss value as the detection score. The Reference (\colorbox{red!10}{Ref}) method~\cite{carlini2021extracting} calibrates the target LLM's perplexity leveraging a similar reference model known or supposed not to have memorized text of interest\footnote{For example for Pythia-12b, a valid reference LLM would be the smaller Pythia-70M.}. Both \colorbox{red!10}{Zlib} and \colorbox{red!10}{Lowercase}~\cite{carlini2021extracting} are also reference-based methods: they utilize zlib compression entropy and lowercased text perplexity as reference for normalization. Lastly, Min-K\%~\cite{shi2023detecting} and Min-K\%++~\cite{zhang2024min} are reference-free methods, which examine token probabilities and average a subset of the minimum token scores, or a function thereof, over the input. For these baselines, we select their hyperparameters by maximizing performance on the validation set(s).

\xhdr{Results} Refer to \Cref{tab:bookmia-results} for results on BookMIA. \ourmethod attains exceptional results, largely surpassing other reference-free approaches. Among these last ones, ours is the only method that can match or outperform even the reference-LLM-based baselines. Importantly, (even partial) access to the TDS reveals crucial to obtain such strong reference-free performance: our ATP-based learnable methods -- which only process features for the actual sequence tokens -- incur indeed significant performance drops.
As for WikiMIA, while full results are enclosed in \Cref{app:Results On The WikiMIA Dataset} (\Cref{tab:WikiMIA}), we point out how \ourmethod consistently surpasses all baselines across all combinations of LLMs and datasets. We also report the second-best method is MinK\%++, followed by MinK\%, consistent with the findings in \cite{zhang2024min}.

\subsection{Generalization to Other LLMs and Datasets}
\label{subsec:Generalization to other LLMs and datasets}
\xhdr{Zero-shot Cross-LLM Generalization in DCD} We assess our model's ability to detect DC in target LLMs that were unseen during training. Using the BookMIA benchmark and the setup described in \Cref{subsec:Data Contamination Detection}, we evaluate our model directly across different LLMs \emph{without any fine-tuning}. This setup is particularly relevant in DCD scenarios where contamination information is not yet available for newly released LLMs. The results are presented in the heatmap shown in \Cref{fig:gene-bookmia}. We observe strong transferability: in 10/12 cases, our model achieves the best performance among reference-free approaches, highlighted in bold in \Cref{fig:gene-bookmia}. Interestingly, in 3/12 cases, \ourmethod (which is reference-free) even surpasses reference-based baselines, as indicated via a superscript of $^*$. We also observe particularly strong transfer across differently sized LLM architectures within the same family and highlight the surprising positive transfer from the largest LlaMa to Pythia models.

\xhdr{Transfer Learning across LLMs and Datasets for HD} Although \ourmethod delivered non-trivial generalization, its zero-shot application on HD was not sufficient to surpass the simpler probability-based techniques. 
This led us to investigate \ourmethod capabilities in a transfer learning setting. Specifically, we fix an LLM/dataset combination and fine-tune the corresponding pretrained \ourmethod either on the remaining LLMs for the same dataset, or the remaining datasets for the same LLM. All Test AUCs of our fine-tuned \ourmethod's are in \Cref{fig:cross-llm-gen-HD-app,fig:cross-data-gen-HD-app}, \Cref{app:Extended Transferability Experiments}, while we report here two representative plots (see~\Cref{fig:transfer_hd_main}). On these heatmaps, superscript `$^*$' indicates the fine-tuned \ourmethod is better than a counterpart trained from scratch in the same setting -- testing for successful transfer; \textbf{bold} indicates it outperforms the best non-learnable probability-based method.

\xhdr{Discussion} First, \ourmethod exhibits solid transferability in both scenarios. The finetuned models consistently outperform their counterparts trained from scratch: 16/18 cases in both the cross-LLM (\Cref{fig:cross-llm-gen-HD-app}, \Cref{app:Extended Transferability Experiments}) and cross-dataset setups (\Cref{fig:cross-data-gen-HD-app}, \Cref{app:Extended Transferability Experiments}) -- see `$^*$' on the off-diagonal entries. This highlights a generally positive transfer of \ourmethod's learned representations across datasets and LLMs, and underscores the suitability of LOS as a data type in capturing generalizable patterns for HD. Second, from a practical perspective, we find that \ourmethod outperforms the best probability-based baseline in 15/18 cases for both the cross-LLM (\Cref{fig:cross-llm-gen-HD-app}) and cross-dataset (\Cref{fig:cross-data-gen-HD-app}) scenarios -- see \textbf{bold} on the off-diagonal entries. Focusing on the IMDB dataset, when training on L-3-8b and testing on Mis-7b (\Cref{fig:transfer_hd_main} (bottom)), our model substantially gains around 27 AUC units over the best probability-based baseline. This result underscores the possibility of transferring across LLMs. A similar trend is observed in the cross-dataset setup (\Cref{fig:transfer_hd_main} (top)): on Mis-7b, when training on HotpotQA or Movies and testing on IMDB, our model achieves a notable improvement of around 30 AUC units compared to the best baseline).

\subsection{Run-Time Analysis} 
\label{subsec:Run-Time}

To empirically assess the efficiency of our approach, we ran a comprehensive set of training and inference timings, reported in full in \Cref{app:Empirical Run-Times} and discussed in the following. The results clearly show \ourmethod features an extremely contained detection latency: $\sim 10^{-5}$s per inference fwd-pass. Training is also efficient, typically completing in under one hour on a single NVIDIA L-40 GPU, and often taking significantly shorter. To contextualize this computational efficiency w.r.t.\ methods relying on multiple prompting/generations \cite{orgad2024llms,kuhn2023semanticuncertaintylinguisticinvariances}, we measured the detection latency of \colorbox{orange!10}{Semantic Entropy} (SE)~\citep{Farquhar2024,kuhn2023semanticuncertaintylinguisticinvariances}, a pioneering method for HD. SE involves generating 10 additional responses and their semantic clustering, operated by checking mutual entailment with an auxiliary language model. On average, we measured $7.14\pm2.97$ seconds per sample detection on Movies/L-3-8b and $7.55\pm1.70$ seconds on Movies/Mis-7b. This is five orders of magnitude slower than \ourmethod, making the latter preferable for both accuracy and latency. Results on other LLM/dataset combinations are reported in \Cref{app:Empirical Run-Times}.

\subsection{The value of \texorpdfstring{$K$}{K} and restricted TDS access}
\label{subsec:k-ablation}
We conclude this section by presenting results on the impact of the parameter $K$, as defined in \Cref{eq: X'}. In particular, we slide $K$ in $\{10, 50, 100, 500, 1000\}$ and discuss here results for the Mis-7B LLM on the HotpotQA dataset, see \Cref{tab:K_ablation_MIS_HQA}. For each of the above values, we measure two quantities: the average probability mass captured, i.e., $\nicefrac{\sum_{n=1}^{N}\left(\sum_{v=1}^{V}(\mathbf{X}'_{n,v})\right)}{N}$, and the corresponding performance of \ourmethod. 
\begin{wraptable}[11]{r}{0.4\linewidth} 
\centering
\footnotesize
\caption{Ablation study on $K$ in \ourmethod. Average Probability Mass (APM) and Test AUC for varying $K$ on Mis-7B - HotpotQA.}
\label{tab:K_ablation_MIS_HQA}
\begin{tabular}{l|cc}
\hline
\textbf{K} & \textbf{APM (\%)} & \textbf{Test AUC} \\
\hline
10   & 99.49 & 71.82 $\pm$ 0.15 \\
50   & 99.80 & 71.87 $\pm$ 0.24 \\
100  & 99.85 & 72.34 $\pm$ 0.20 \\
500  & 99.99 & 72.67 $\pm$ 0.32 \\
1000 & 99.99 & \textbf{72.92} $\pm$ 0.45 \\
\hline
\end{tabular}
\end{wraptable}
We observe that the former exceeds $0.99$ for all considered $K$'s, indicating that even values as small as $K=10$ are often sufficient to convey most of the information in the full TDS. In terms of performance, Test AUC tends to improve as $K$ values increase, though with diminishing returns beyond $K \geq 100$. As expected, the value $K = 1000$ tends to deliver the best performance overall, this confirmed by results on the other LLM-dataset combinations reported in \Cref{app:Ablation Study}. Given its extremely contained run-times (see above), this value appears to hit a sweet-spot optimizing performance and complexity. 
Most notably, however, even with $K = 10$, \ourmethod outperforms all baselines and matches the white-box \textit{Activation Probe}, highlighting the practical effectiveness of \ourmethod even in API-limited settings such as in GPT models -- at the time of writing, exposing only the top $K = 20$ output logits.

\section{Conclusions}\label{sec:conclusions}
We proposed \ourmethod, an efficient method to detect data contamination and hallucinations in LLMs by leveraging their output signatures (LOS), defined as the union of Token Distribution Sequences (TDS) and Actual Token Probabilities (ATP). \ourmethod\ consists of a lightweight attention-based model operating on an effective encoding of the LOS. We proved it unifies and extends existing gray-box methods under a general framework, and experimentally showed it outperforms state-of-the-art gray-box methods across datasets and LLMs. It also exhibited promising generalization and transfer capabilities of \ourmethod, both across datasets and across LLMs. Our framework could be applied to other tasks, such as detecting LLM-generated content. Additional sources of information can also be incorporated, e.g., in the absence of latency constraints, it can be interesting to include ``exact-token'' flags as proposed by~\cite{orgad2024llms}. Last, the LOS can be extended to account for multiple prompting~\citep{kuhn2023semanticuncertaintylinguisticinvariances}.

\clearpage

\subsubsection*{Acknowledgements}
The authors are grateful to Beatrice Bevilacqua for insightful discussions. G.B.\ is supported by the Jacobs Qualcomm PhD Fellowship. F.F.\ conducted this work supported by an Aly Kaufman and an Andrew and Erna Finci Viterbi Post-Doctoral Fellowship. Y.G.\ is supported by
the UKRI Engineering and Physical Sciences Research Council (EPSRC) CDT in Autonomous and Intelligent Machines and Systems (grant reference EP/S024050/1). H.M.\ is a Robert J.\ Shillman Fellow and is supported by the Israel Science Foundation through a personal grant (ISF 264/23) and an equipment grant (ISF 532/23). D.L.\ is funded by an NSF Graduate Fellowship. Research was also supported by the Israeli Ministry of Science, Israel-Singapore binational grant 207606. F.F.\ is extremely grateful to the members of the ``Eva Project'', whose support he immensely appreciates.

\bibliographystyle{plainnat}
\bibliography{main}

\clearpage

\appendix

\section{Proofs}
\label{app:Proofs}
\ThresholdImplementation*

\begin{proof}
    We define $\mathcal{D} \coloneqq \mathcal{X} \times \mathcal{M}.$
    Recall that the target function we want to approximate is the gated scoring function $R$ as defined in \Cref{eq:GSF}, which can be written as follows:
    \begin{equation}
        R(x) = \sum_{i=1}^{N_\text{max}} \mathbb{I}(\kappa(x)_i \geq T(x)) \cdot g(x)_i, \label{eq:target_function}
    \end{equation}
    for $x \in \mathcal{D}$.

    Define $f^{(1)} : \mathcal{D} \rightarrow \mathbb{R}^{N_\text{max}}$ to be the components of the sum in \Cref{eq:target_function}:
    \begin{equation}
        f^{(1)}(x)_i = \mathbb{I}(\kappa(x)_i \geq T(x)) \cdot g(x)_i.
    \end{equation}

    It follows that $R(x) = \sum_{i=1}^{N_{\text{max}}} f^{(1)}(x)_i$.

    \textbf{Step 1:} We begin by selecting \(K = V_\text{max}\) as a hyperparameter\footnote{For LLMs with a vocabulary size smaller than \(V_\text{max}\), appropriate padding can be applied.} and initializing the parameters \(\mathbf{p}_1\), \(\mathbf{p}_2\), and \(\mathbf{W}\) as follows:
    \begin{align}
        \mathbf{p}_1 &= 0, \\
        \mathbf{p}_2 &= 1, \\
        \mathbf{W} &= I_{K \times K}.
    \end{align}
    As a result, the input to the transformer encoder in our architecture (see \Cref{eq: arch}) becomes \(\mathbf{X}' || \mathbf{p} \in \mathbb{R}^{N_\text{max} \times (V_\text{max} + 1)}\).
    
    This simplifies our architecture in \Cref{eq: arch} to:
    \begin{equation}
        h_\theta (\mathbf{X}, \mathbf{p}) = \mathcal{T} (\mathbf{X}' || \mathbf{p}).
    \end{equation}

    \textbf{Step 2: $\mathbf{f^{(1)} \in L^1(\mathcal{D})}$.}  
    Define the $L^1(\mathcal{D})$ norm for a field $\mathcal{F}: \mathcal D \to \mathbb{R}^{n_2}$ as:
    \begin{align}
        \norm{\mathcal{F}}_{L^1} & = \int_{x \in \mathcal{D}} \norm{\mathcal{F}(x)}_{1} \, dx = 
        \int_{x \in \mathcal{D}} \sum_{i=1}^{n_2}{|\mathcal{F}(x)_i|} \, dx  \nonumber \\ 
        & =
        \sum_{i=1}^{n_2} \int_{x \in \mathcal{D}} |\mathcal{F}(x)_i| \, dx  =
        \sum_{i=1}^{n_2} \norm{\mathcal{F}(x)_i}_{L^1},
    \end{align}
    where $\|v\|_1 = \sum_{i=1}^{n_2} |v_i|$ is the $l_1$ norm of the vector $v$. 
    
    Next, observe that $f^{(1)} \in L^1(\mathcal{D})$. To see this, first note that $f^{(1)}$ is measurable. The indicator function is measurable because the indicator set is the preimage of the measurable function $\kappa(x) - T(x)$ on the closed set $[0, \infty)$. Thus, $f^{(1)}$, being a product of measurable functions, is measurable.
    Next, we show that the $L^1$ norm is finite. This is true because $f^{(1)}$ is a product of the integrable function $g$ and the bounded function $\mathbf{1}$ on the compact domain $\mathcal D$. 

    \textbf{Step 3: Approximating $\mathbf{f^{(1)}}$ by a continuous field $\mathbf{\tilde{f}^{(1)}}$.}
    We need to approximate the field $f^{(1)} : \mathcal{D} \rightarrow \mathbb{R}^{N_\text{max}}$ by a continuous field, so that we can apply existing results on approximating continuous functions with Transformers. We state the following Lemma, saying the continuous fields are dense in $L_1(\mathcal{D})$.

    \begin{lemma}
        For any $g \in L^1(\mathcal D)$ and any $\epsilon > 0$, there exists a continuous $\tilde g \in L^1(\mathcal D)$ such that $\norm{g - \tilde g}_{L^1} < \epsilon$.
    \end{lemma}
    \begin{proof}
        Consider the coordinate functions $g_i: \mathcal D \to \RR$. Since continuous functions are dense in $L^1$ for scalar valued functions, we can choose continuous $\tilde g_i$ such that $\norm{g - \tilde g}_{L^1} < \epsilon/N$. Thus, letting $\tilde g(x) = [g_1(x), \ldots, g_N(x)] \in \RR^N$, it holds that $\norm{g - \tilde g} = \sum_{i=1}^N \norm{g_i - \tilde g_i} < \epsilon$.
    \end{proof}
    Thus, we can choose a function $\tilde f^{(1)}$ such that,
    \begin{equation}
        \norm{f^{(1)} - \tilde f^{(1)}} < \frac{\epsilon}{2N_{\text{max}}}.
    \end{equation}

    \textbf{Step 4: Approximating the continuous field $\mathbf{\tilde{f}^{(1)}}$ by a transformer model $\mathbf{h^{(1)}_\theta}$.}
    We start by restating the following from \cite{yun2019transformers} in our context,
    \begin{theorem}
    \label{thm: transApp}
    Let $1 \leq p < \infty$ and $\epsilon > 0$, then for any given $f \in \mathcal{F_\text{CD}}$, where $\mathcal{F_\text{CD}}$ is the set of all continuous functions that map a compact domain in $\mathbb{R}^{n \times d}$ to $\mathbb{R}^{n \times d} $,  there exists a Transformer network (with positional encodings) $g: \mathbb{R}^{n \times d} \rightarrow \mathbb{R}^{n \times d}$ such that we have $\norm{f-g}_{L^p} \leq \epsilon$.
    \end{theorem}

    To apply this theorem in our context, we observe that in our case $d \coloneqq V_\text{max}+1$ and $n \coloneqq N_\text{max}$ for the input space, and the domain $\mathcal{D} \subseteq \mathbb{R}^{N_\text{max} \times (V_\text{max} +1)}$ is compact. Thus $\tilde{f}^{(1)} \in \mathcal{F_\text{CD}}$ (note that the output space dimension in our case is $\RR^{N_\text{max} \times 1}$ instead of $\RR^{N_\text{max} \times d}$, but this can be handled using zero-padding). Using $p=1$, it holds that there exists a transformer $h_\theta^{(1)} $ s.t., $\norm{h_\theta^{(1)} - \tilde{f}^{(1)}} < \frac{\epsilon}{2N_\text{max}}$. 

    \textbf{Step 5: Pooling.} 
    Our model concludes with a [CLS] token pooling mechanism, which is equivalent in expressiveness to the standard sum pooling method. Thus, assuming that the final layer of our model is given by $h^{(1)}_\theta(x)$, our model can be written as follows,

    \begin{equation}
        h_\theta(x) =    \sum_{i=1}^{N_\text{max}} \left( h^{(1)}_\theta(x)_{i} \right).
    \end{equation}

    \textbf{Step 6: Approximating the objective function.} Intuitively, $h_\theta(x)$ approximates $R(x)$ because $h^{(1)}_\theta(x)_i$ approximates $f^{(1)}(x)_i$.

We demonstrate this as follows.

    \begin{align}
        \norm{h_{\theta} - R}_{L_1} &= \norm{\sum_{i=1}^{N_{\text{max}}}  \left( h^{(1)}_{\theta;i} \right) - \sum_{i=1}^{N_{\text{max}}} f^{(1)}_i }_{L_1}\\
        & \leq \sum_{i=1}^{N_{\text{max}}} \norm{h^{(1)}_{\theta; i} - f^{(1)}_i}\\
        & = \sum_{i=1}^{N_{\text{max}}} \norm{h^{(1)}_{\theta; i} + (\tilde{f}^{(1)}_{i} - \tilde{f}^{(1)}_{i}) - f^{(1)}_i}  \\
        & \leq \sum_{i=1}^{N_{\text{max}}} \norm{h^{(1)}_{\theta; i} - \tilde{f}^{(1)}_{i}} + \sum_{i=1}^{N_{\text{max}}} \norm{\tilde{f}^{(1)}_{i} - f^{(1)}_i}
    \end{align}
    We applied the triangle inequality to obtain the two inequalities. Next, note that for a field $\mathcal F: \RR^{n_1} \to \RR^{n_2}$, the $L^1$ norm of any coordinate function is less than the $L^1$ norm of $\mathcal F$: $\norm{\mathcal F_j}_{L^1} \leq \norm{\mathcal F}_{L^1}$ for any $j \in \{1, \ldots, n_2\}$. This can be seen directly from the definition of the $L^1$ norm of $\mathcal F$. Combining this with our choices of $\tilde f$ and $h_\theta$ shows that:
\begin{align}
        & \sum_{i=1}^N \norm{h^{(1)}_{\theta; i} - \tilde{f}^{(1)}_{i}} + \sum_{i=1}^{N_{\text{max}}} \norm{\tilde{f}^{(1)}_{i} - f^{(1)}_i} \\
        & <  \sum_{i=1}^{N_{\text{max}}} \frac{\epsilon}{2N_\text{max}} + \sum_{i=1}^{N_{\text{max}}} \frac{\epsilon}{2N_\text{max}}\\
        & = \epsilon.
\end{align}
In total, this means that $\norm{h_{\theta} - R}_{L_1} < \epsilon$, so we are done.
\end{proof}

\GSFsCaptureBaselines*
\begin{proof}
    We will prove the Proposition by defining, for each baseline, the functions implementing components $\kappa, T, g$, assuming no ties in the ATP values $\mathbf{p}$.

    \textbf{Mean Aggregated Probability.} This baseline simply outputs the mean across the ATPs $\mathbf{p}$.
    The following selection of functions implements it as a GFS:
    $$
        \kappa(\mathbf{X}', \mathbf{p}) = \mathbf{1} \quad
        T(\mathbf{X}', \mathbf{p}) = 0 \quad
        g(\mathbf{X}', \mathbf{p}) = \frac{1}{N} \mathbf{p}
    $$

    \textbf{Min Aggregated Probability} outputs the min value across the ATPs $\mathbf{p}$.
    The following selection of functions implements it as a GFS:
    $$
        \kappa(\mathbf{X}', \mathbf{p}) = - \mathbf{p} \quad
        T(\mathbf{X}', \mathbf{p}) = - \min(\mathbf{p}) \quad
        g(\mathbf{X}', \mathbf{p}) = \mathbf{p}
    $$

    \textbf{Max Aggregated Probability} outputs the max value across the ATPs $\mathbf{p}$.
    We simply pick:
    $$
        \kappa(\mathbf{X}', \mathbf{p}) = \mathbf{p} \quad
        T(\mathbf{X}', \mathbf{p}) = \max(\mathbf{p}) \quad
        g(\mathbf{X}', \mathbf{p}) = \mathbf{p}
    $$

    \textbf{MinK\%.} Please refer to Section~\ref{subsec:TDSnet provably generalizes previous approaches}.

    \textbf{MinK\%++.} Let $\bar{\mathbf{p}} = \frac{\log(\mathbf{p}) - \boldsymbol{\mu}}{\boldsymbol{\sigma}}$, be the normalized version of $\mathbf{p}$, with:
    \begin{align}
        \boldsymbol{\mu}_i &= \mathbb{E}_{\mathbf{X}'_i} [\log(\mathbf{X}'_i)] = \sum_{v=1}^{V} \mathbf{X}'_{i,v} \cdot \log(\mathbf{X}'_{i,v}), \nonumber \\
        \boldsymbol{\sigma}_i &= \sqrt{\mathbb{E}_{\mathbf{X}'_i} [(\log(\mathbf{X}'_i) - \boldsymbol{\mu}_i)^2]} \nonumber \\
        &= \sqrt{\sum_{v=1}^{V} \mathbf{X}'_{i,v} \cdot \big(\log(\mathbf{X}'_{i,v}) - \boldsymbol{\mu}_i\big)^2},
    \end{align}
    Where $\mathbf{X}'$ is given from \Cref{eq: X'}.

    The baseline is implemented by setting:
    \begin{align}
        T(\mathbf{X}', \mathbf{p}) &= - \text{perc}(\bar{\mathbf{p}}, K) = - \big(\text{sort}(\bar{\mathbf{p}})_{\left\lceil \frac{K}{100} \cdot N \right\rceil} \big), \nonumber \\
        \kappa(\mathbf{X}', \mathbf{p}) &= -\bar{\mathbf{p}}, \quad g(\mathbf{X}', \mathbf{p}) = \frac{\bar{\mathbf{p}}}{\left\lceil \frac{K}{100} \cdot N \right\rceil}. \nonumber
    \end{align}

    \textbf{Loss as a Privacy Proxy \cite{yeom2018privacy}.} This method uses the model's negated loss as a proxy for contamination, which can be defined as the average of the log ATPs. The method can thus be implemented with:
    \begin{align}
        \kappa(\mathbf{X}', \mathbf{p}) = \mathbf{1}, \quad T(\mathbf{X}', \mathbf{p}) = 0, \quad g(\mathbf{X}', \mathbf{p}) = \frac{1}{N}\log(\mathbf{p}).
    \end{align}
\end{proof}

\BaselinesApprox*
\begin{proof}
    
    To prove \Cref{prop:BaselinesApprox}, it suffices to show the following. First (i), that the baselines can be implemented as in \Cref{eq:GSF}, given their sequence length and vocabulary size satisfy, $N \leq N_{\text{max}}, \quad V \leq V_{\text{max}}$, where values in the inputs for indices larger than $N, V$ are `padded' with e.g., $-1$. Second (ii), that their implementations are realized with $\kappa$, $T$, and $g$ which are all measurable, and with $g$ also integrable.
    
    \textbf{(i)} Let us slightly modify the implementations provided in the Proof for~\Cref{prop:GSFs_capture_baselines} to correctly account for padding values. Let us conveniently define:
    \begin{align}
        \alpha: \mathbb{R} \rightarrow \mathbb{R}, &~~~ \alpha(x) = 1 - \mathrm{ReLU}(-x) = \begin{cases}
            1 & x \geq 0 \\
            1 + x & x < 0
        \end{cases} \nonumber \\
        N_{\text{eff}} = \sum_{i=1}^{N_{\text{max}}} \alpha(\mathbf{p}_i) &\quad V_{\text{eff}} = \sum_{v=1}^{V_{\text{max}}} \alpha(\mathbf{X}_{1,v})
    \end{align}
    \noindent as well as the following function, which will help us `manipulate' the padding value in order not to interfere with the effective computations required by baselines:
    \begin{align}
        \beta: \mathbb{R} \rightarrow \mathbb{R}, &\quad \beta(x; M,f) = \begin{cases}
            f(x) & x \geq 0 \\
            M & x = -1
        \end{cases}, M  > 0.
    \end{align}

    \textbf{Mean Aggregated Probability.}
    $$
        \kappa(\mathbf{X}', \mathbf{p}) = \mathbf{1} \quad
        T(\mathbf{X}', \mathbf{p}) = 0 \quad
        g(\mathbf{X}', \mathbf{p}) = \frac{1}{N_{\text{eff}}} \mathbf{p} \circ \alpha(\mathbf{p}),
    $$
    where $\circ$ denotes the hadamard (element-wise) product.

    \textbf{Min Aggregated Probability.}
    \begin{align*}
        \kappa(\mathbf{X}', \mathbf{p}) &= - \beta(\mathbf{p}) \\
        T(\mathbf{X}', \mathbf{p}) &= - \min(\beta(\mathbf{p})) \\
        g(\mathbf{X}', \mathbf{p}) &= \mathbf{p} \\
        M &= 2, \\ 
        f &\equiv \text{id.}
    \end{align*}

    \textbf{Max Aggregated Probability.}
    $$
        \kappa(\mathbf{X}', \mathbf{p}) = \mathbf{p} \quad
        T(\mathbf{X}', \mathbf{p}) = \max(\mathbf{p}) \quad
        g(\mathbf{X}', \mathbf{p}) = \mathbf{p}
    $$

    \textbf{MinK\%.}
    \begin{align*}
        \kappa(\mathbf{X}', \mathbf{p}) &= - \beta(\mathbf{p}) \\
        T(\mathbf{X}', \mathbf{p}) &= - \big(\text{sort}(\beta(\mathbf{p}))_{\left\lceil \frac{K}{100} \cdot N_{\text{eff}} \right\rceil} \big) \\
        g(\mathbf{X}', \mathbf{p}) &= \frac{\log(\beta(\mathbf{p}))}{\left\lceil \frac{K}{100} \cdot N_{\text{eff}} \right\rceil} \\
        M &= 2, \\ f &\equiv \text{id.}
\end{align*}
    where the note the application of $\beta$ inside the $\log$ prevents negative inputs.

    \textbf{MinK\%++.} Before illustrating how this baseline is implemented, we note the following. In order for the normalization of log-probs to be well-defined, it is required that: (1) $\mathbf{\mu}$ is finite, (2) the denominator is greater than $0$. As for (1), we note that null probability values ($X_{i,v}=0$) would be problematic, as they would cause the $\log$ function to output $-\infty$. We assume, in this case, that all probability values lie in $[\epsilon_1, 1]$, with $\epsilon_1$ being a small value such that $0 < \epsilon_1 < 1$. Regarding (2), we see that the problematic situation would occur in cases where the probability distribution is uniform. We assume to handle this case by adding a small positive constant $\epsilon_2 > 0$ in the denominator, so that the normalization would take the form: $\bar{\mathbf{p}} = \frac{\log(\mathbf{p}) - \boldsymbol{\mu}}{\boldsymbol{\sigma} + \epsilon_2}$.
    
    Under these assumptions, we define the following $\beta$ functions:
    \begin{align*}
        \beta_1 &= \beta(\cdot; 2, \text{id.}) \\ \beta^{i}_{2} &= \beta(\cdot; -\frac{2\log(\epsilon_1)}{\epsilon_2}, f_i), \\ 
        f_i(x) &= \frac{\log(x) - \boldsymbol{\mu}_i}{\left\lceil \frac{K}{100} \cdot N_{\text{eff}} \right\rceil \boldsymbol{\sigma}_i + \epsilon_2}
    \end{align*}
    where we note that $-\frac{2\log(\epsilon_1)}{\epsilon_2}$ upper-bounds all the possible values that can be attained by $f_i$'s under our assumptions.

    At this point, we observe that the values $\boldsymbol{\mu}_i, \boldsymbol{\sigma}_i$ can be correctly obtained as follows, in a way that is not influenced by our padding scheme:
    \begin{align}
        \boldsymbol{\mu_i} &= \sum_{v} \alpha(\mathbf{X}'_{i, v}) \cdot \mathbf{X}'_{i,v} \log\left(\beta_1(\mathbf{X}'_{{iv}})\right) \\
        \boldsymbol{\sigma_i} &= \sqrt{\sum_v \alpha(\mathbf{X}'_{i, v}) \cdot \mathbf{X}'_{i, v}\left(\log(\beta_1(\mathbf{X}')_{i, v}) - \boldsymbol{\mu}_i\right)^2}
    \end{align}

    At this point, let $\boldsymbol{\beta}_2(\mathbf{p})_i = \beta^{i}_2(\mathbf{p}_i)$. We set:
    \begin{align*}
        \kappa(\mathbf{X}', \mathbf{p}) &= - \boldsymbol{\beta}_2(\mathbf{p}) \\
        T(\mathbf{X}', \mathbf{p}) &= - \big(\text{sort}(\boldsymbol{\beta}_2(\mathbf{p}))_{\left\lceil \frac{K}{100} \cdot N_{\text{eff}} \right\rceil} \big) \\
        g(\mathbf{X}', \mathbf{p}) &= \frac{\boldsymbol{\beta}_2(\mathbf{p})}{\left\lceil \frac{K}{100} \cdot N_{\text{eff}} \right\rceil}
    \end{align*}
    \noindent and note that the $K$-th percentile in $T$ is correctly computed despite the padding values due to the specific choice of $M$ in $\beta_2$'s.

    \textbf{Loss as a Privacy Proxy \cite{yeom2018privacy}.} 
        \begin{align}
        \kappa(\mathbf{X}', \mathbf{p}) = \mathbf{1}, \quad T(\mathbf{X}', \mathbf{p}) = 0, \quad g(\mathbf{X}', \mathbf{p}) = \frac{1}{N_\text{eff}}\log(\mathbf{p}).
    \end{align}

    \textbf{(ii)} We now proceed to show that the implementations above are obtained via measurable functions $\kappa$, $T$, and a measurable and integrable function $g$, which completes the proof.  
    
    \textbf{Step 1:} Consider a fixed sequence length $N' \in [N_{\text{max}}]$ and a fixed vocabulary size $V \in [V_{\text{max}}]$. When restricted to these parameters, all relevant functions are continuous. This follows from the fact that each function, when restricted in this manner, is composed of continuous functions.  
    
    \textbf{Step 2:} The input domain for each combination of sequence length $N' \in [N_{\text{max}}]$ and vocabulary size $V \in [V_{\text{max}}]$ forms a compact set, and the union of all of this domains is also compact (as a finite union of compact sets). Moreover, for any two distinct pairs $(N_1, V_1)$ and $(N_2, V_2)$, if either $N_1 \neq N_2$ or $V_1 \neq V_2$, then the corresponding domains are disjoint. 
    
    In most of our cases of interest, this follows from the fact that probabilities lie within $[0,1]$ and that padding is represented by $-1$. In other cases, e.g., the application of $\beta$, the sets might be different, but remain disjoint and compact.
    
    Thus, by the following lemma, all functions $\kappa, T, g$ for all baselines are continuous, completing the proof.

\begin{restatable}[]{lemma}{}\label{lem:general_continuity}
Let \( X \) be a subset of a metric space, which is compact, and can be expressed as a finite disjoint union of compact subsets \( X_{i} \) indexed by a finite set \( I \), i.e., 
\[
X = \bigsqcup_{i \in I} X_i.
\]

Suppose a function \( f: X \to \mathbb{R}^{n} \) is defined such that for each \( i \in I \), there is a continuous function
\[
g^{(i)}: X_i \to \mathbb{R}^{n}
\]
satisfying \( f|_{X_i} = g^{(i)} \). Then, \( f \) is continuous on \( X \).
\end{restatable}

The finite disjoint union of compact subsets correspond to all possible sequence lengths ($N' \in N_\text{max}$) and vocabulary sizes ($V' \in V_\text{max}$).
Below we provide the proof for \Cref{lem:general_continuity}.

\begin{proof}
Consider any point \( \mathbf{x} \in X \), and let \( (\mathbf{x}^{(m)}) \) be a sequence converging to \( \mathbf{x} \), in $X$. We need to show that 
\[
f(\mathbf{x}^{(m)}) \to f(\mathbf{x}) \quad \text{as } m \to \infty.
\]
Since \( X \) is a finite disjoint union of compact subsets \( X_i \), there exists an index \( i^* \) such that \( \mathbf{x} \in X_{i^*} \). 

Since the subsets \( X_i \) are disjoint and compact, there exists a positive minimum separation distance between distinct subsets, defined as,
\[
\delta^* = \frac{1}{2} \min_{i \neq j} \inf_{\mathbf{x} \in X_i, \mathbf{y} \in X_j} \|\mathbf{x} - \mathbf{y}\|.
\]
Since each \( X_i \) is compact and the index set is finite\footnote{\url{https://proofwiki.org/wiki/Distance\_between\_Disjoint\_Compact\_Set\_and\_Closed\_Set\_in\_Metric\_Space\_is\_Positive\#google\_vignette}}, this minimum distance is well-defined and strictly positive.

Because \( \mathbf{x}^{(m)} \to \mathbf{x} \), there exists an integer \( M \) such that for all \( m > M \), we have
\[
\|\mathbf{x}^{(m)} - \mathbf{x}\| < \delta^*.
\]
By the definition of \( \delta^* \), this ensures that for sufficiently large \( m \), the sequence \( \mathbf{x}^{(m)} \) remains in \( X_{i^*} \), i.e., \( \mathbf{x}^{(m)} \in X_{i^*} \) for all \( m > M \).

Since \( f \) coincides with \( g^{(i^*)} \) on \( X_{i^*} \), we have
\[
f(\mathbf{x}^{(m)}) = g^{(i^*)}(\mathbf{x}^{(m)}), \quad \text{for all } m > M.
\]
By assumption, \( g^{(i^*)} \) is continuous on \( X_{i^*} \), so
\[
g^{(i^*)}(\mathbf{x}^{(m)}) \to g^{(i^*)}(\mathbf{x}) \quad \text{as } m \to \infty.
\]
Since \( f(\mathbf{x}) = g^{(i^*)}(\mathbf{x}) \), it follows that
\[
f(\mathbf{x}^{(m)}) \to f(\mathbf{x}),
\]
which proves that \( f \) is continuous at \( \mathbf{x} \). Since \( \mathbf{x} \) was arbitrary, \( f \) is continuous on \( X \).
\end{proof}

\end{proof}

\section{Extended Experimental Section}
\label{app:exp}

\subsection{Experimental Details}
Our experiments were conducted using the PyTorch \cite{paszke2019pytorch} framework (License: BSD), using a single NVIDIA L-40 GPU for all experiments regarding \ourmethod. We use a fixed batch size of 64 for all the tasks and datasets, and a fixed value of 8 heads (except for the Movies\cite{orgad2024llms} dataset) in our light-weight transformer encoder for \ourmethod. Hyperparameter tuning was performed utilizing the Weight and Biases framework \cite{wandb} -- see \Cref{tab: Hyperparameters}.

\subsection{HyperParameters}
\label{app: HyperParameters}
In this section, we detail the hyperparameter search conducted for our experiments. We use the same hyperparameter grid for our main model, \ourmethod, and our proposed learning-based baselines, namely, \textsc{ATP+R-MLP}, \textsc{ATP+R-Transf.}. Additionally, we note that for a given dataset, we maintained the same grid search approach for all LLMs' LOSs that we have trained on. The hyperparameter search configurations for all datasets are presented in \Cref{tab: Hyperparameters}. The grid search optimizes for the AUC calculated on the validation set.

\begin{table*}[ht]
    \centering
  \caption{Hyperparameter search grid for \ourmethod.}
    \label{tab: Hyperparameters}
    \resizebox{1\textwidth}{!}{%
    \begin{tabular}{l|c|c|c|c|c|c} 
    \toprule
        Dataset & Num. layers & Learning rate & Embedding size & Epochs  & Dropout & Weight Decay  \\  
        \midrule 
        \textsc{HotpotQA} & $\{ 1,2 \}$ & $\{ 0.0001 \}$ & $\{ 128, 256 \}$ & $\{ 300 \}$  & $\{ 0, 0.3 \}$ & $\{ 0, 0.001 \}$ \\
        \textsc{IMDB} & $\{ 1,2 \}$ & $\{ 0.0001 \}$ & $\{ 128, 256 \}$ & $\{ 300 \}$  & $\{ 0, 0.3 \}$ & $\{ 0, 0.001 \}$ \\
        \textsc{Movies} & $\{1, 2\}$ & $\{ 0.0001 \}$ & $\{ 128, 256 \}$ & $\{ 300, 500 \}$ & $\{ 0.0, 0.3, 0.5 \}$  & $\{ 0, ,0.001, 0.005 \}$\\
        \midrule
        \textsc{WikiMIA (32/64)} & $\{ 1,2 \}$ & $\{ 0.0001 \}$ & $ \{ 128,256 \}$ & $ \{ 100, 500, 1000 \}$ & $\{ 0, 0.3 \}$ & $\{ 0, 0.001 \}$  \\
        \textsc{BookMIA} & $\{ 1,2 \}$ & $\{ 0.0001 \}$ &  $ \{ 64, 128 \}$ & $ \{ 500 \}$ & $ \{ 0, 0.3, 0.5 \}$ & $\{ 0, 0.001 \}$  \\
        \bottomrule
    \end{tabular}
    }
\end{table*}

\subsection{Optimizers and Schedulers} For all datasets we employ the AdamW optimizer \cite{loshchilov2017decoupled} paired with a Linear scheduler, using a warm up of 10\% of the epochs. We apply an early stopping criterion if there is no improvement in validation performance for 30 consecutive epochs.

\subsection{Our Baselines and Rank Encoding}
\label{app:our-baselines}

\xhdr{Rank Encoding}
We construct the following learnable rank encoding~\footnote{For the Wiki-MIA dataset, we used a lookup table for Rank Encoding, where the index corresponds to $r_i$ and the value is an embedding.},
\begin{equation}
\label{eq:PE}
    \text{RE}(\mathbf{X}, \mathbf{p}) = \mathbf{p} \odot \mathbf{r}^{\text{scaled}} \cdot \mathbf{w}_1 + \mathbf{p} \cdot \mathbf{w}_2,
\end{equation}
where $\odot$ is the hadamard product, and $\mathbf{w}_1, \mathbf{w}_2$ are learnable parameters in $\mathbb{R}^{d}$. As a result, $\text{RE}(\mathbf{X}, \mathbf{p})$ is in $\mathbb{R}^{N \times d}$. Importantly, the multiplication by $\mathbf{p}$ makes sure that the rank encoding and the TDS are in similar scales, especially when using log probabilities or logits.

\xhdr{Our baselines}
Below we present our additional learnable baselines. \textbf{ATP+R-Transf} is implemented as described in \Cref{eq: arch}, but without incorporating the TDS (\( \mathbf{X} \)), as follows:
\begin{equation}
    h_\theta(\mathbf{X}, \mathbf{p}) = \mathcal{T} \left(\text{RE}(\mathbf{X}, \mathbf{p})\right),
\end{equation}
where \( \mathcal{T} \) represents an encoder-only transformer architecture \cite{vaswani2017attention}. \textbf{ATP+R-MLP} is similar to \textbf{ATP+R-Transf.} but replaces the transformer with an MLP. Formally:
\begin{equation}
    h_\theta(\mathbf{X}, \mathbf{p}) = \text{MLP} \left(\text{RE}(\mathbf{X}, \mathbf{p})\right),
\end{equation}

\subsection{Dataset Description}

\subsubsection{Datasets for Hallucination Detection}
\label{app:Datasets for Hallucination Detection}
In this section, we provide an overview of the three datasets used in our hallucination detection analysis; we mostly follow the framework given in \cite{orgad2024llms} in constructing the datasets. Our aim was to ensure coverage of a wide variety of tasks, required reasoning skills, and dataset diversity. For each dataset, we highlight its unique contributions and how it complements the others.

For all datasets, we used a consistent split of 10,000 training samples and 10,000 test samples.

\begin{enumerate}

    \item \textbf{HotpotQA} \cite{yang2018hotpotqa} (License: CC-BY-SA-4.0): This dataset is specifically designed for multi-hop question answering and includes diverse questions that require reasoning across multiple pieces of information. Each entry comprises supporting Wikipedia documents that aid in answering the questions. For our analysis, we utilized the ``without context'' setting, where questions are posed directly. This setup demands both factual knowledge and reasoning skills to generate accurate answers. 

    \item \textbf{Movies} \cite{orgad2024llms} (License: MIT): This dataset checks for factual accuracy in scenarios regarding movies. LLMs are asked, in particular, who was the actor/actress playing a specific role in a movie of interest. This dataset contains 7857 test samples.

    \item \textbf{IMDB} (originally released with no known license by \citet{maas2011learning}): This dataset contains movie reviews designed for sentiment classification tasks. Following the approach outlined in \cite{orgad2024llms}, we applied a one-shot prompt to guide the large language model (LLM) in using the predefined sentiment labels effectively.

\end{enumerate}

\textbf{Annotation collection for HD.} Specifically, following \cite{orgad2024llms}, the dataset \(D = \{(q_i, z_i)\}_{i=1}^{\ell} \) contains \( {\ell} \) question-answer pairs, where \( q_i \) are questions and \( z_i \) are ground-truth answers. For each \( q_i \), the model generates a response \( \hat{z}_i \), with predicted answers \( \{\hat{z}_i\}_{i=1}^{\ell} \). The LOS for each response, \( \{ (\mathbf{X}, \mathbf{p})_i \}_{i=1}^{\ell} \), is saved. Correctness labels \( y_i \in \{0, 1\} \) are assigned by comparing \( \hat{z}_i \) to \( z_i \), resulting in the error-detection dataset \( \{ (\mathbf{X}, \mathbf{p})_i, y_i \}_{i=1}^\ell \). 

\textbf{LLMs.} We consider the following LLMs for our experiments on HD:
\begin{enumerate}

    \item \textbf{Mistral-7b-instruct-v0.2} \cite{jiang2023mistral} (License: Apache-2.0). Referred to as Mis-7b in the main text and accessed through the Hugging Face interface at \url{https://huggingface.co/mistralai/Mistral-7B-Instruct-v0.3}.

    \item \textbf{Llama-3-8b-Instruct} \citep{touvron2023llama} (License: Llama-3\footnote{\url{https://huggingface.co/meta-llama/Meta-Llama-3-8B/blob/main/LICENSE}}). Referred to as L-3-8b in the main text and accessed through the Hugging Face interface at \url{https://huggingface.co/meta-llama/Meta-Llama-3-8B-Instruct3}.

    \item \textbf{Qwen-2.5-7b-Instruct} (License: Apache-2.0): Referred to as Q-2.5-7b in the main text and accessed through the Hugging Face interface at \url{https://huggingface.co/Qwen/Qwen2.5-7B-Instruct}.

\end{enumerate}

\subsubsection{Datasets for Data Contamination Detection}
\label{app:bookmia}
\textbf{BookMIA.} \citep{shi2023detecting} The original BookMIA data have been obtained from the Hugging Face dataset \texttt{swj0419/BookMIA}\footnote{\url{https://huggingface.co/datasets/swj0419/BookMIA}.}, accessed via the Hugging Face python \texttt{datasets} API (License: MIT). The dataset totals $9,870$ excerpts from a total of $100$ books, of which $50$ are labeled as members (positives) and $50$ are labeled as non-members (negatives).

Throughout all experiments on BookMIA, including the evaluation of baselines, we process only the first $128$ words from each excerpt, originally $512$-word long. This expedient allowed for faster LLM inference and lighter data storage at the time of dataset creation, i.e., the extraction and saving of LLM outputs.

As no standard split is available for this dataset, we proceed by randomly forming \emph{training} and \emph{test} sets in the proportions of, resp., $80\%$ and $20\%$. To ensure that \emph{all} excerpts from the same book are in either one of the two sets (and never in both), we first separate books into two separate lists based on their label, shuffle the obtained lists using a random seed of \texttt{42}, and then, for each of the two lists, take the first $80\%$ of books as training books, and the remaining $20\%$ as test books. Training and test sets are obtained by taking the corresponding excerpts from, respectively, training and test books. After this, we verified that the obtained sets are both approximately class-balanced ($\approx 50\%$ of excerpts in both the training and test sets are labeled as positives).

In the case of the reference-based baseline, we consider the smallest-sized available counterparts for the respectively attacked LLMs, namely: Pythia 70M for Pythia models and Llama-1 7B for Llama models. All LLMs are accessed through the Hugging Face python interface, specifically: \texttt{EleutherAI/pythia-70m}, \texttt{EleutherAI/pythia-\{6.9,12\}b}\footnote{\url{https://huggingface.co/EleutherAI/pythia-70m} (License: Apache-2.0), \url{https://huggingface.co/EleutherAI/pythia-6.9b}, \url{https://huggingface.co/EleutherAI/pythia-12b}.} and \texttt{huggyllama/llama-\{7,13,30\}b}\footnote{\url{https://huggingface.co/huggyllama/llama-7b}, \url{https://huggingface.co/huggyllama/llama-13b}, \url{https://huggingface.co/huggyllama/llama-30b}.} (License: Llama\footnote{\url{https://huggingface.co/huggyllama/llama-13b/blob/main/LICENSE}, \url{https://huggingface.co/huggyllama/llama-30b/blob/main/LICENSE}}).

\textbf{WikiMIA.} WikiMIA\cite{shi2023detecting} (License: MIT) is the first benchmark for pre-training data detection, comprising texts from Wikipedia events. The distinction between training and non-training data is determined by timestamps. WikiMIA organizes data into splits based on sentence length, enabling fine-grained evaluation. It also considers two settings: original and paraphrased. The original setting evaluates the detection of verbatim training texts, while the paraphrased setting, where training texts are rewritten using ChatGPT, assesses detection on paraphrased inputs. In this paper, we consider the original (non-paraphrased) split and focus on the 32 and 64 split sizes, as they contain the largest number of samples, approximately 750 and 550, respectively. On top of the LLMs attacked in BookMIA, here we also attack Mamba-1.4b (License: Apache-2.0), accessed via the Hugging Face interface (\url{https://huggingface.co/state-spaces/mamba-1.4b}).

\subsection{Extended Transferability Experiments}
\label{app:Extended Transferability Experiments}

\xhdr{Setup} We fine-tune the \ourmethod\ given target LLM/dataset (depending on the task) for 10 epochs ---significantly fewer than the number of epochs used in our standard training protocol. Notably, this fine-tuning process takes less than one minute in practice, on a single NVIDIA L-40 GPU.

To evaluate the effectiveness of fine-tuning, we benchmark the resulting model against two baselines. First, to assess knowledge transfer, we compare it with a \ourmethod\ model trained from scratch under the same 10-epoch setup. Second, we compare against the strongest known probability-based baselines with comparable detection latency—specifically, the strongest among Logits-mean/min/max and Probas-mean/min/max (see \Cref{tab:hall-merged}). This comparison is essential: generalization scores above 0.5 AUC are only meaningful if they outperform non-learnable baselines that rely purely on probabilities or logits. We note that we exclude  \colorbox{orange!10}{P(true)} and \colorbox{orange!10}{Semantic Entropy}  baselines from this assessment, as they incur significantly higher latency (five order of magnitude higher than \ourmethod) and are thus not directly comparable to \ourmethod\. These methods require additional generation or prompting to detect hallucinations. For a more detailed analysis, please refer to \Cref{subsec:Run-Time}.

Comprehensive results are shown in \Cref{fig:cross-llm-gen-HD-app} (cross-LLM generalization) and \Cref{fig:cross-data-gen-HD-app} (cross-dataset generalization).

In the heatmaps, a superscript `$^*$' indicates that the fine-tuned \ourmethod{} outperforms its scratch-trained counterpart in the same setting---evidence of successful transfer. \textbf{Bold} entries denote cases where it surpasses the best non-learnable probability-based method.

\begin{figure*}[t]
    \centering
    \includegraphics[width=0.329\textwidth]{Arxiv_updated/figures/Heatmaps/Cross_models/IMDB.pdf}
    \includegraphics[width=0.329\textwidth]{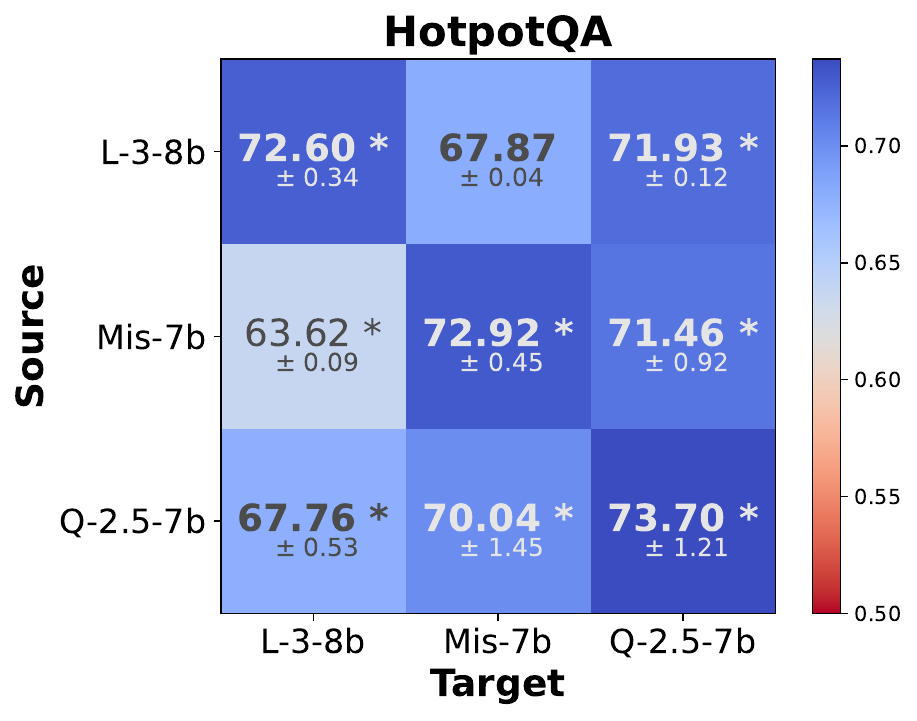}
    \includegraphics[width=0.329\textwidth]{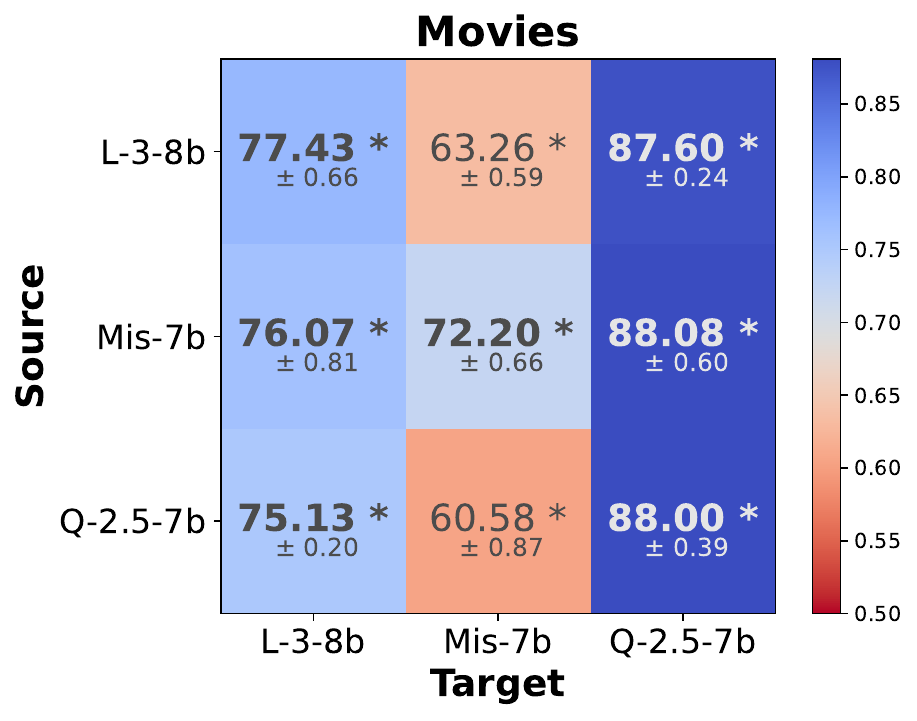}
    \caption{Cross-LLM transfer Test AUCs (cols: source LLMs, rows: target LLMs). \textbf{Bold}: finetuning \ourmethod\ outperforms baselines, $^*$: it outperforms the same \ourmethod\ trained from scratch.
    }
    \label{fig:cross-llm-gen-HD-app}
\end{figure*}

\begin{figure*}[t]
    \centering
    \includegraphics[width=0.329\textwidth]{Arxiv_updated/figures/Heatmaps/Cross_datasets/Mis-7b.pdf}
    \includegraphics[width=0.329\textwidth]{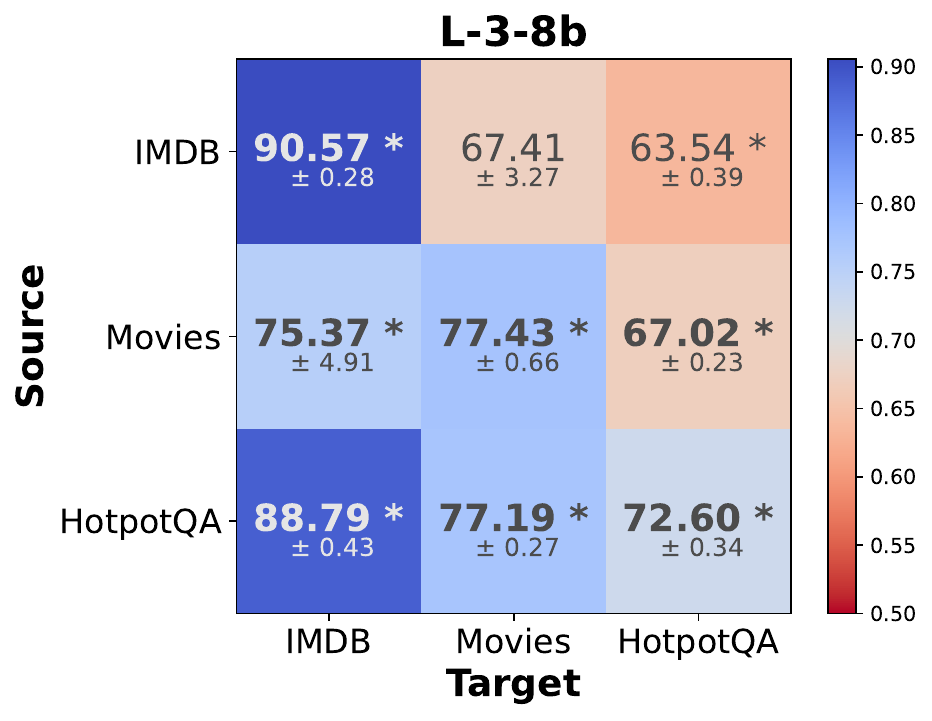}
    \includegraphics[width=0.329\textwidth]{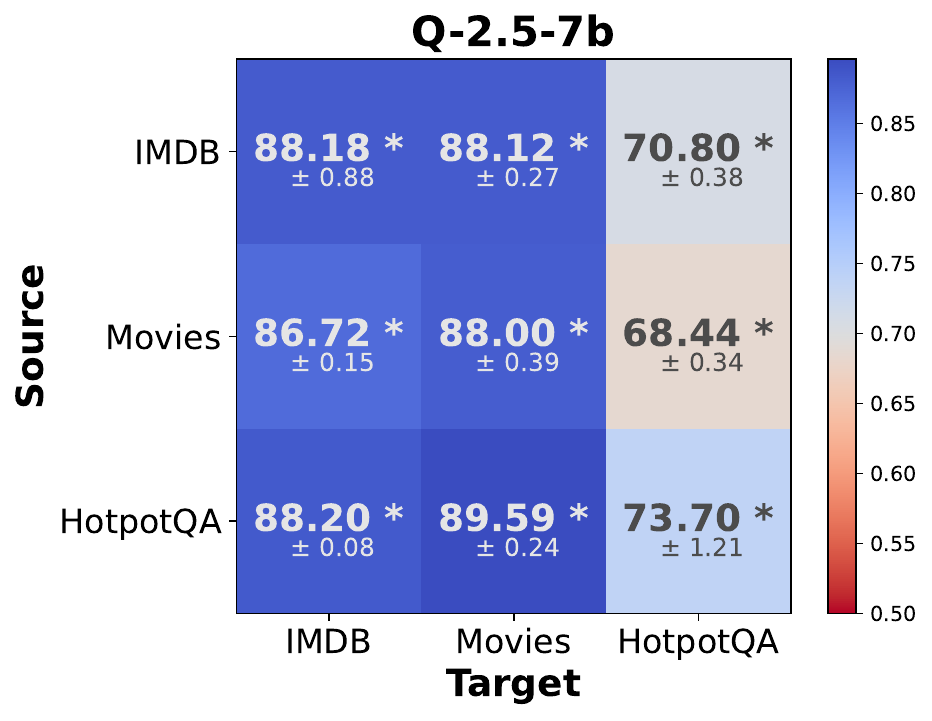}
    \caption{Cross-dataset transfer Test AUCs (cols: source data, rows: target data). \textbf{Bold}: finetuning \ourmethod\ outperforms baselines, $^*$: it outperforms the same \ourmethod\ trained from scratch.}
    \label{fig:cross-data-gen-HD-app}
\end{figure*}

\subsection{Ablation Study}
\label{app:Ablation Study}
Existing methods often overlook a critical aspect of LOS. Specifically, they primarily rely on the ATP, $\mathbf{p}$, while neglecting the TDS, $\mathbf{X}$. In this subsection, we conduct an ablation study to evaluate the significance of the TDS in general, as well as its size, namely the hyperparameter $K$ introduced in \Cref{eq: X'}.

\begin{figure*}[t]
    \centering
    \includegraphics[width=\textwidth]{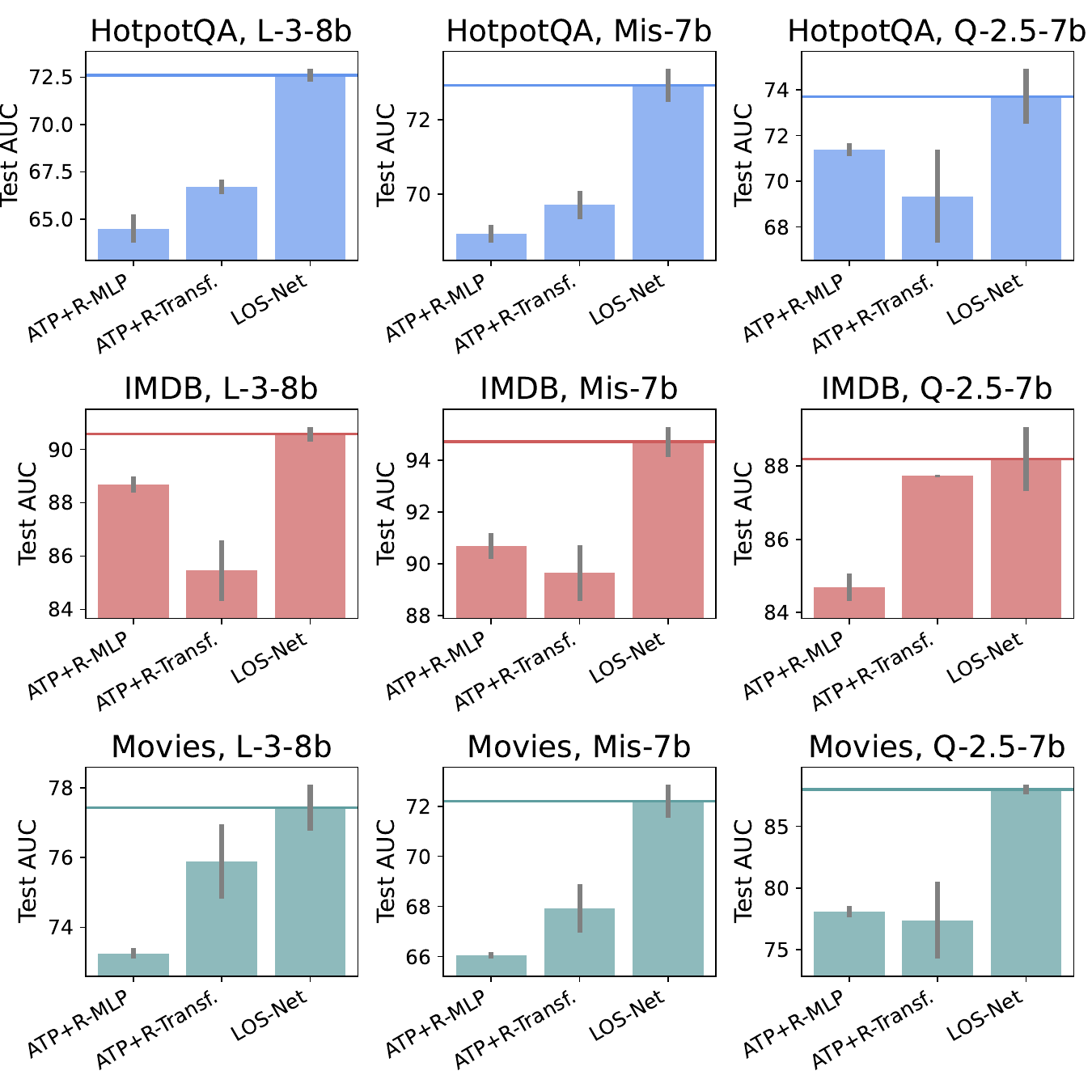}
    \caption{Ablation study evaluating the role of the TDS (\( \mathbf{X} \)) and the ATP (\( \mathbf{p} \)) on our HD setups, including datasets HotpotQA, IMDB, Movies, and LLMs L-3-8b, Mis-7b, Q-2.5-7b.}
    \label{fig:ablation_on_three_arch}
\end{figure*}

\begin{figure*}[t]
    \centering
    \includegraphics[width=\textwidth]{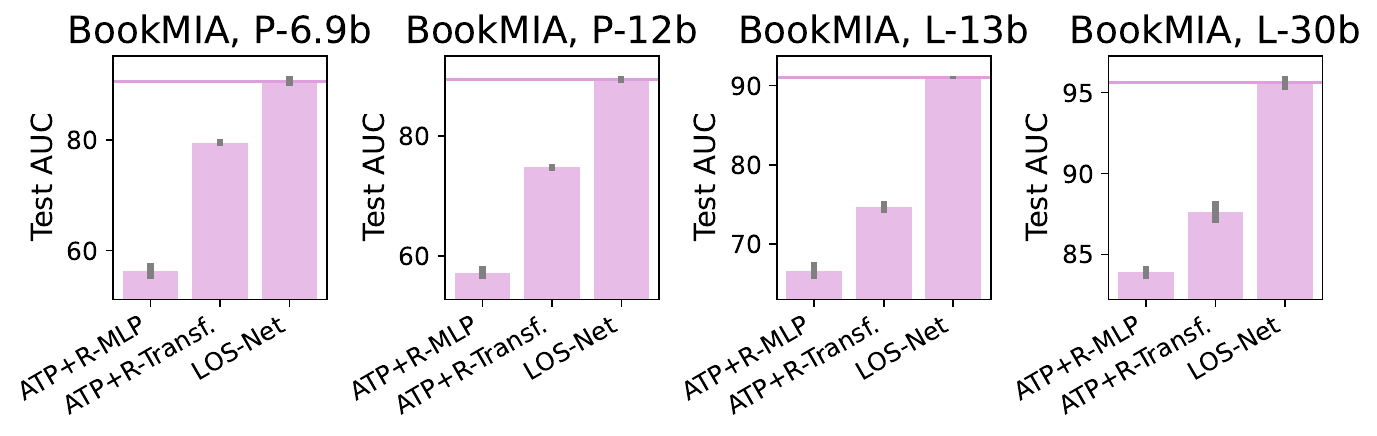}
    \caption{Ablation study evaluating the role of the TDS (\( \mathbf{X} \)) and the ATP (\( \mathbf{p} \)) on BookMIA for Pythia and Llama-1 LLMs.}
    \label{fig:ablation_on_three_arch_bookmia}
\end{figure*}

\textbf{The Role of the TDS (\( \mathbf{X} \)).}  
As a case study, we examine both the DCD task on the BookMIA dataset and the HD task across the three datasets: HotpotQA, IMDB, and Movies. 
%
\Cref{fig:ablation_on_three_arch,fig:ablation_on_three_arch_bookmia} report a close-up comparison between \ourmethod\ and our two proposed baselines, which explicitly neglect the TDS, namely, \textsc{ATP+R-Transf.} and \textsc{ATP+R-MLP}. These plots consistently show how the best-performing model is \ourmethod. In many cases, \ourmethod\ outperforms the alternatives by a significant margin, indicating that the information encoded in the TDS (\( \mathbf{X} \)) is crucial for both tasks. 
Regarding the two ATP-based baselines, we report that \textsc{ATP+R-Transf.} obtains better performance than \textsc{ATP+R-MLP} in 8 out of 12 cases, but these improvements do not seem to follow a clear pattern across LLMs and datasets. The only exception is BookMIA, on which the former architecture outperformed the latter across all the four attacked LLMs.

\begin{figure*}[!ht]
    \centering
    \includegraphics[width=\textwidth
    ]{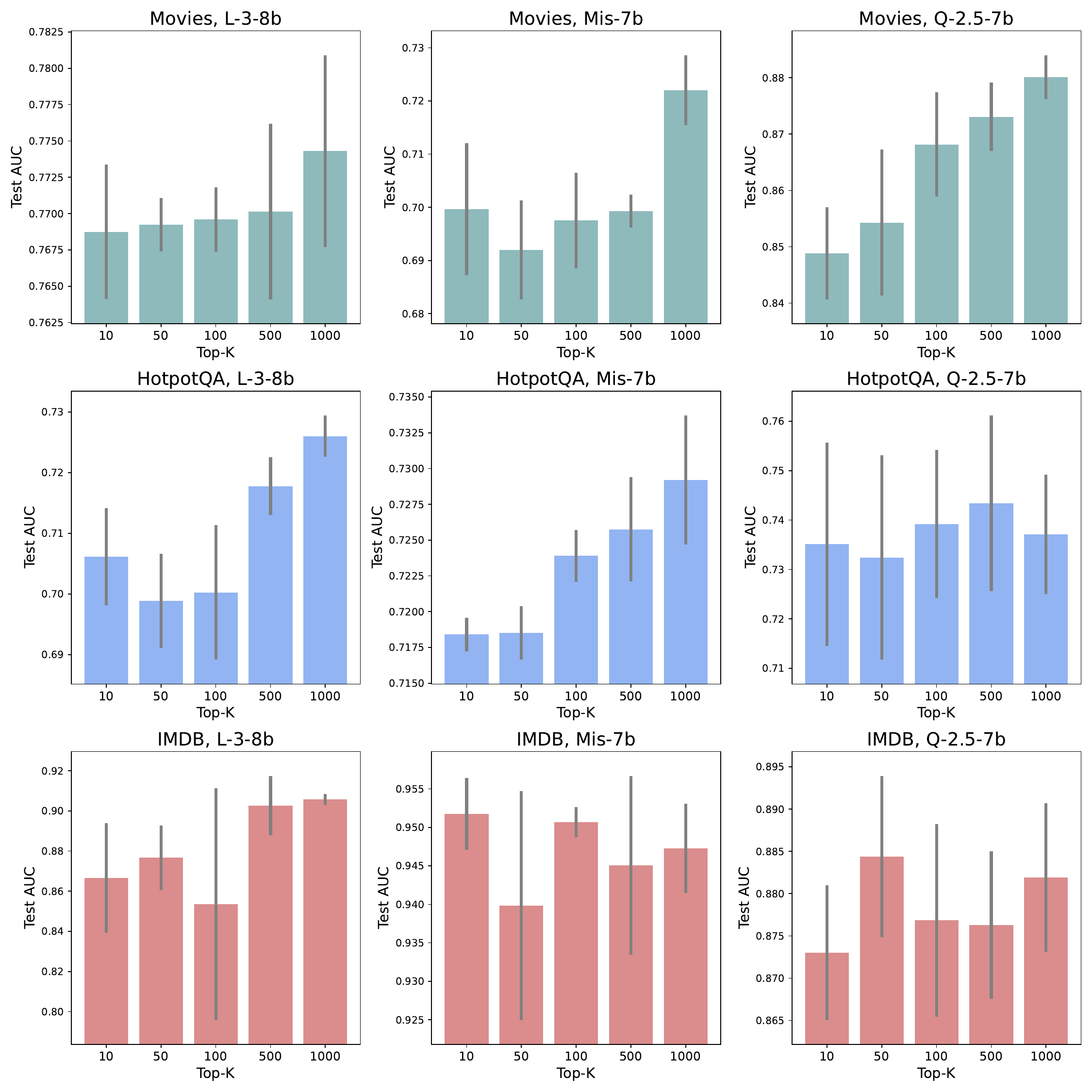}
    \caption{Ablation study analyzing the effect of the hyperparameter \( K \) introduced in \Cref{eq: X'}.}
    \label{fig:Ablation_histogram}
\end{figure*}

\begin{figure*}[htbp]
    \centering
    \begin{subfigure}[b]{0.3\textwidth}
        \includegraphics[width=\textwidth]{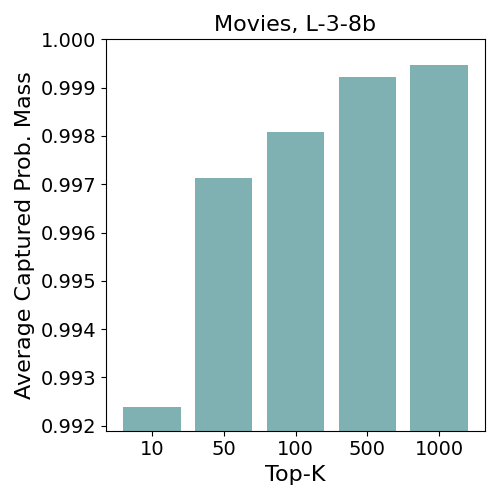}
    \end{subfigure}\hfill
    \begin{subfigure}[b]{0.3\textwidth}
        \includegraphics[width=\textwidth]{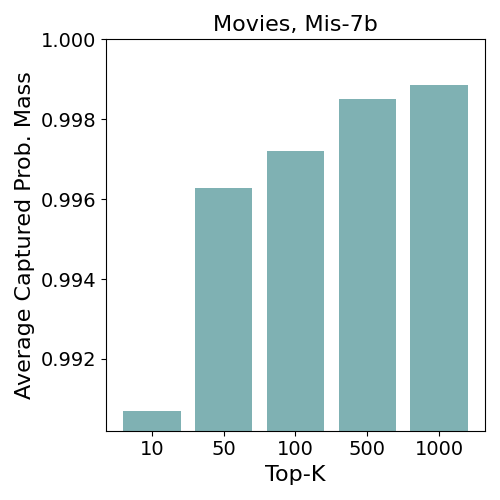}
    \end{subfigure}\hfill
    \begin{subfigure}[b]{0.3\textwidth}
        \includegraphics[width=\textwidth]{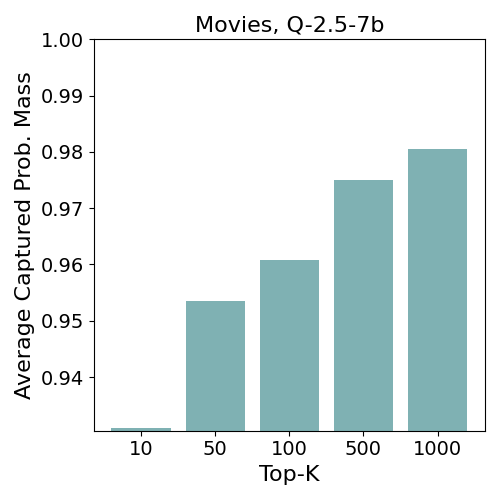}
    \end{subfigure}

    \begin{subfigure}[b]{0.3\textwidth}
        \includegraphics[width=\textwidth]{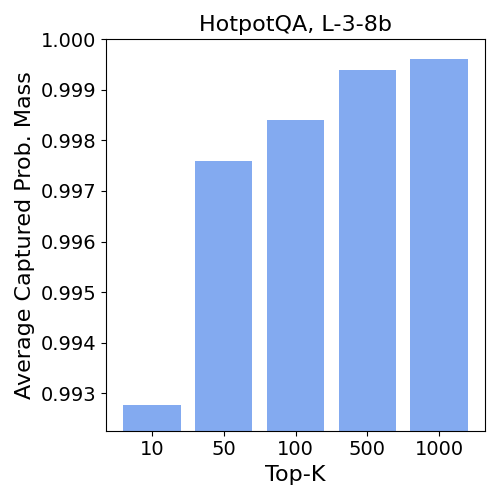}
    \end{subfigure}\hfill
    \begin{subfigure}[b]{0.3\textwidth}
        \includegraphics[width=\textwidth]{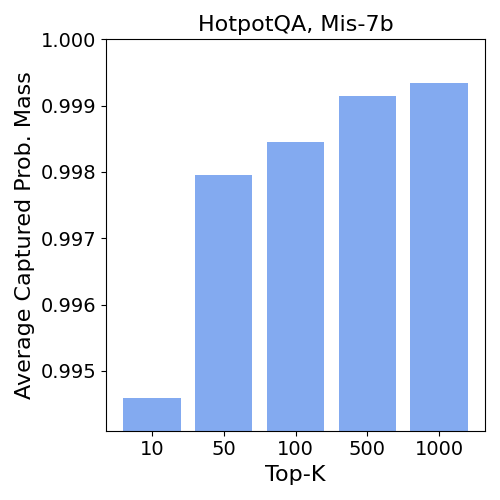}
    \end{subfigure}\hfill
    \begin{subfigure}[b]{0.3\textwidth}
        \includegraphics[width=\textwidth]{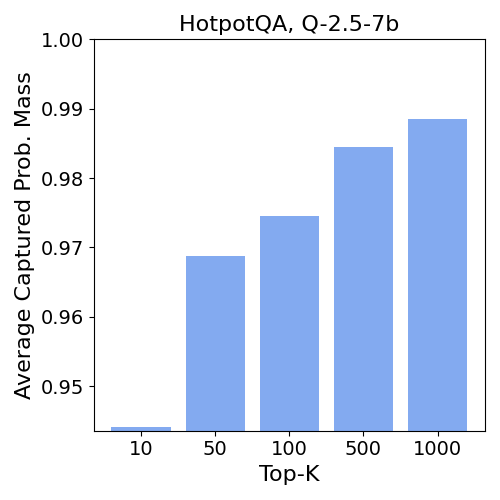}
    \end{subfigure}

    \begin{subfigure}[b]{0.3\textwidth}
        \includegraphics[width=\textwidth]{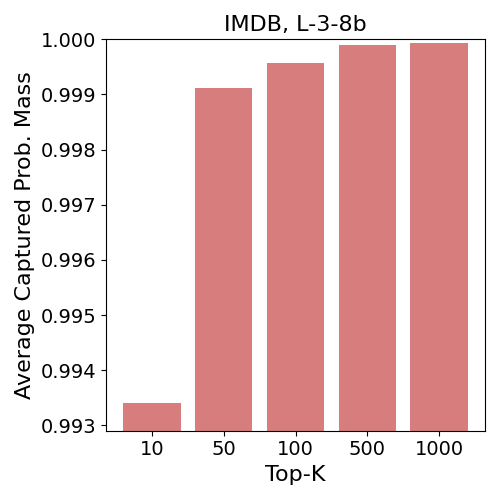}
    \end{subfigure}\hfill
    \begin{subfigure}[b]{0.3\textwidth}
        \includegraphics[width=\textwidth]{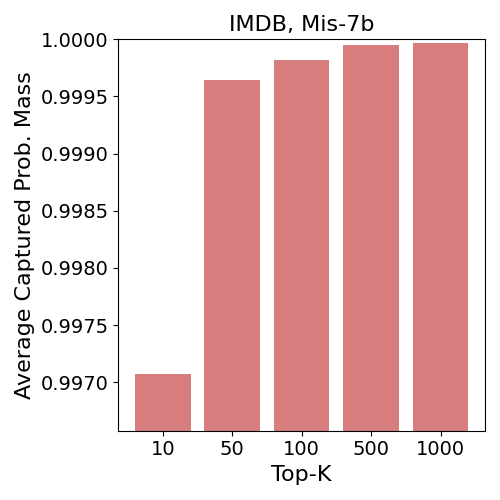}
    \end{subfigure}\hfill
    \begin{subfigure}[b]{0.3\textwidth}
        \includegraphics[width=\textwidth]{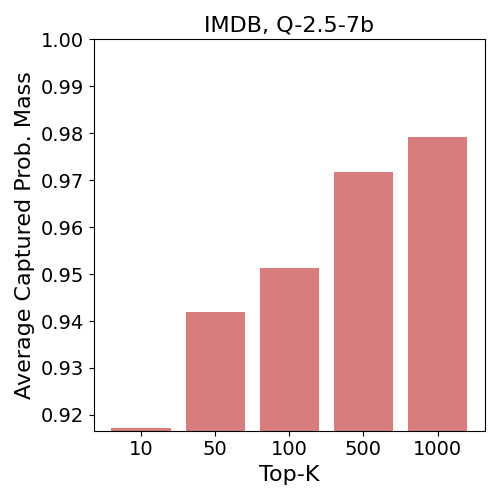}
    \end{subfigure}

\caption{Probability mass ($y$-axis) as a function of $K$ ($x$-axis) for each LLM/dataset combination considered in the HD study.}

    \label{fig:prob_mass_full}
\end{figure*}

\textbf {The hyperparameter \( K \).} To evaluate the impact of the hyperparameter \( K \) introduced in \Cref{eq: X'}, we conduct a comprehensive case study focusing on the task of HD.

We experiment with various values of \( K \), specifically \( K \in \{10, 50, 100, 500, 1000\} \), and trained the same selected model whose results are reported in~\Cref{tab:hall-merged} for $K=1000$. The corresponding Test AUCs are presented in \Cref{fig:Ablation_histogram}.

From the reported bar plots, we do observe that performances either weakly increase with $K$ (see, e.g., Movies for Q-2.5-7b or HotpotQA on L-3-8b), or stay approximately constant (see, e.g., IMDB on Mis-7b). In any case, the performance difference w.r.t.\ our default setting $K=1000$ remains contained. This is a valuable feature, as it unlocks the effective application of \ourmethod\ even on non fully open LLMs such as the most recent models released by OpenAI\footnote{At the time of writing, OpenAI's API only gives access to a maximum of $20$ top scoring logprobs (\url{https://platform.openai.com/docs/api-reference/completions/create}, accessed May 2025.}.

To complement \Cref{fig:Ablation_histogram}, we present \Cref{fig:prob_mass_full}, which shows the average probability mass captured for each value of $K$, computed as $\nicefrac{\sum_{n=1}^{N}\left(\sum_{v=1}^{V}(\mathbf{X}'_{n,v})\right)}{N}$. Across all LLM/dataset combinations and for every $K \in \{10, 50, 100, 500, 1000\}$, the captured probability mass exceeds 91\%. This helps explain why even small values such as $K=10$ perform well, as observed in \Cref{fig:Ablation_histogram}.


\begin{table}[h]
\footnotesize
    \centering
    \caption{Test AUC scores for HD on Qwen-2.5-7b-Instruct (Q-2.5-7b). The best-performing method is in \textbf{bold}, and the second best is \underline{underlined}.}
    \begin{tabular}{l|ccc}
        \toprule
        \multirow{2}{*}{Method} & HotpotQA & IMDB & Movies \\
         \cmidrule{2-4}
        & \multicolumn{3}{c}{Q-2.5-7b} \\
        \midrule
        Logits-mean          & 
            66.2 & 
            74.8 & 
            71.3 \\
        Logits-min           & 
            59.8 & 
            72.1 & 
            42.1 \\
        Logits-max           & 
            60.4 & 
            60.7 & 
            65.1 \\
        Probas-mean          & 
            67.5 & 
            74.6 & 
            74.2 \\
        Probas-min           & 
            54.4 & 
            65.4 & 
            44.7 \\
        Probas-max           & 
            61.8 & 
            50.1 & 
            72.9 \\
        \midrule
        \textbf{\textsc{ATP+R-MLP}}           & 
            \underline{71.38}  \scalebox{0.8}{$\pm$ 0.28}  &
            84.69 \scalebox{0.8}{$\pm$ 0.37}  &
            \underline{78.06} \scalebox{0.8}{$\pm$ 0.45}  \\
        \textbf{\textsc{ATP+R-Transf.}}       & 
            69.34 \scalebox{0.8}{$\pm$ 2.04}  &
            \underline{87.73} \scalebox{0.8}{$\pm$ 0.03} &
            77.37 \scalebox{0.8}{$\pm$ 3.13}   \\
        \textbf{\ourmethod}      & 
            \textbf{73.71} \scalebox{0.8}{$\pm$ 1.21} &
            \textbf{88.19} \scalebox{0.8}{$\pm$ 0.88}  &
            \textbf{88.00} \scalebox{0.8}{$\pm$ 0.39}  \\
        \bottomrule
    \end{tabular}
    \label{tab:qwen}
\end{table}

\subsection{Results For Hallucination Detection for Qwen-2.5-7b}
\label{app:Results On Qwen}

\Cref{tab:qwen} reports results on our three considered HD datasets over LLM Qwen-2.5-7b-Instruct (Q-2.5-7b)~\citep{qwen2}. We can see \ourmethod outperforms all non-learnable output-based baselines by large margin, as well as our learnable baselines \textsc{ATP+R-Transf.} and \textsc{ATP+R-MLP}.

\subsection{Results On The WikiMIA Dataset}
\label{app:Results On The WikiMIA Dataset}

\begin{table*}[ht]
    \centering
    \scriptsize
    \caption{Comparison of AUC over four different LLMs, on DCD, over the discussed baselines methods. The best-performing method is in \textbf{bold}, and the second best is \underline{underlined}. \colorbox{red!10}{Reference-based} approaches are shaded in pink.}
    \label{tab:WikiMIA}
    \resizebox{\textwidth}{!}{
    \begin{tabular}{l|cccc|cccc}
        \toprule
        Dataset $\rightarrow$ & \multicolumn{4}{c|}{WikiMIA - 32} & \multicolumn{4}{c}{WikiMIA - 64} \\
        \cmidrule{2-9}
        LLM $\rightarrow$ & P-6.9b & L-13b & L-30b & M-1.4b & P-6.9b & L-13b & L-30b & M-1.4b\\
        \midrule
        Loss                 & 63.82 \tiny$\pm$2.22  & 67.45 \tiny$\pm$1.57  & 69.37 \tiny$\pm$2.66  & 60.89 \tiny$\pm$1.35 & 60.59 \tiny$\pm$3.50  & 63.68 \tiny$\pm$5.57  & 66.18 \tiny$\pm$4.64  & 58.46 \tiny$\pm$3.69  \\
        MinK            & 66.39 \tiny$\pm$2.56  & 68.08 \tiny$\pm$1.45  & 70.02 \tiny$\pm$2.92  & 63.27 \tiny$\pm$1.85 & 65.07 \tiny$\pm$1.80  & 66.24 \tiny$\pm$5.01  & 68.45 \tiny$\pm$4.11  & 62.46 \tiny$\pm$2.75 \\
        MinK++          & \underline{70.60} \tiny$\pm$3.58  & \underline{84.93} \tiny$\pm$1.76  & \underline{84.46} \tiny$\pm$1.43  & \underline{67.06} \tiny$\pm$2.78 & \underline{71.82} \tiny$\pm$3.73  & \underline{85.66} \tiny$\pm$2.25  & \underline{85.02} \tiny$\pm$2.79  & \underline{67.24} \tiny$\pm$4.06 \\
        \midrule
        \rowcolor{red!10} Zlib                  & 64.35 \tiny$\pm$3.46  &67.70 \tiny$\pm$2.25  & 69.81 \tiny$\pm$3.17  & 62.07 \tiny$\pm$3.35  &  62.59 \tiny$\pm$3.38 & 65.40 \tiny$\pm$5.35  & 67.61 \tiny$\pm$4.21  & 60.59 \tiny$\pm$3.73  \\
        \rowcolor{red!10} Lowercase           & 62.09 \tiny$\pm$4.22  & 64.03 \tiny$\pm$6.97  & 64.31 \tiny$\pm$5.18 & 60.59 \tiny$\pm$3.24 & 58.34 \tiny$\pm$4.21  & 62.63 \tiny$\pm$5.05  & 61.54 \tiny$\pm$7.81  & 57.03 \tiny$\pm$2.83 \\
        \rowcolor{red!10} Ref             & 63.45 \tiny$\pm$6.03  & 57.77 \tiny$\pm$5.94  & 63.55 \tiny$\pm$6.69  & 62.05 \tiny$\pm$5.43 & 62.35 \tiny$\pm$4.84  & 63.07 \tiny$\pm$5.09  & 68.94 \tiny$\pm$5.83  & 60.29 \tiny$\pm$4.62 \\
        \midrule
        \textbf{\ourmethod}      & \textbf{76.98} \tiny$\pm$3.36  & \textbf{93.46} \tiny$\pm$1.31  & \textbf{93.76} \tiny$\pm$1.56 & \textbf{71.04} \tiny$\pm$9.07 &  \textbf{76.00} \tiny$\pm$5.48  & \textbf{87.86} \tiny$\pm$3.73  & \textbf{93.04} \tiny$\pm$2.51 & \textbf{79.39} \tiny$\pm$2.61  \\        
        \bottomrule
    \end{tabular}}
\end{table*}

The WikiMIA-32 and -64 datasets contain excerpts from Wikipedia articles, consisting of, resp., 32 and 64 words. The distinction between contaminated and uncontaminated data is determined by timestamps. As in~\citep{shi2023detecting,zhang2024min}, we attack Mamba-1.4b~\cite{gu2023mamba} (M-1.4b), LlaMa-13b/30b~\cite{touvron2023llama} (L-13b/30b), Pythia-6.9b \cite{biderman2023pythia} (P-6.9b).

Since WikiMIA does not provide an official training split and our method requires labeled data, we perform 5-fold cross-validation with training, validation, and testing splits\footnote{We use $\{\frac{3}{5}, \frac{1}{5}, \frac{1}{5}\}$ as the ratios for training, validation, and testing, respectively.} and rerun all baselines under the same protocol for a fair comparison. Results are reported as the mean and standard deviation across folds. For these datasets only, setting the hyperparameter $K=1000$ (recall \Cref{eq: X'}) led to suboptimal performance in preliminary experiments, thus, we set $K=\text{``Full-Vocabulary''}$.

As shown in \Cref{tab:WikiMIA}, \ourmethod\ consistently surpasses all baseline methods across all eight combinations of LLMs and datasets. Notably, for L-30b, our model achieves an AUC score that is more than 8 points higher than the best-performing baseline, MinK\%++ for both datasets, demonstrating a substantial improvement. Similarly, for P-6.9b, our model maintains a steady advantage of approximately 5 AUC for both datasets, further underscoring its robustness. Overall, the second-best method is MinK\%++, followed by MinK\%, consistent with the findings of \cite{zhang2024min}.

\subsection{Empirical Run-Times}
\label{app:Empirical Run-Times}
In \Cref{tab:LOS-Net full timing}, we report the wall-clock training times (for the best selected model based on the held-out validation set) and single-example detection times for \ourmethod\, for all experiments presented in this paper.

\begin{table*}[t!]
\centering
\small 
\setlength{\tabcolsep}{5pt} 
\captionsetup{width=\textwidth}  
\caption{Comparison of training and detection times of our model \ourmethod\, across all the DC settings explored in our paper, as well as HD settings for Mis-7b and L-3-8b. All are measured on a \textbf{single NVIDIA L-40 GPU}.}
\makebox[\textwidth][c]{%
\begin{tabular}{
  >{\bfseries}l
  >{\raggedright\arraybackslash}p{2cm}
  >{\raggedright\arraybackslash}p{2cm}
  >{\centering\arraybackslash}p{3.8cm}
  >{\centering\arraybackslash}p{4cm}
}
\toprule
\rowcolor{lightgray}
\textbf{Task} & \textbf{Target LLM} & \textbf{Dataset} & \textbf{Training Time} 
[h = hours, m = minutes, 
s = seconds] & \textbf{Detection Time (Mean $\pm$ Std)} [seconds] \\
\midrule
\multirow{6}{*}{HD} & 
\multirow{3}{*}{Mis-7b} & HotpotQA & 9m 19s & $3.32 \times 10^{-5} \pm 1.20 \times 10^{-5}$ s \\
 & & IMDB & 16m 8s & $4.05 \times 10^{-5} \pm 1.83 \times 10^{-5}$ s \\
 & & Movies & 17m 50s & $1.95 \times 10^{-5} \pm 7.24 \times 10^{-6}$ s \\
\cmidrule{2-5}
 & \multirow{3}{*}{L-3-8b} & HotpotQA & 6m 39s & $2.38 \times 10^{-5} \pm 7.18 \times 10^{-6}$ s \\
 & & IMDB & 4m 23s & $3.37 \times 10^{-5} \pm 1.53 \times 10^{-5}$ s \\
 & & Movies & 11m 34s & $3.05 \times 10^{-5} \pm 1.21 \times 10^{-5}$ s \\
\midrule
\multirow{11}{*}{DCD} & 
\multirow{3}{*}{L-13b} & WikiMIA-32 & 33m 6s & $4.13 \times 10^{-5} \pm 1.67 \times 10^{-6}$ s \\
 & & WikiMIA-64  & 2m 7s & $2.67 \times 10^{-5} \pm 1.12 \times 10^{-5}$ s \\
 & & BookMIA & 7m 32s & $3.67 \times 10^{-5} \pm 8.65 \times 10^{-6}$ s \\
\cmidrule{2-5}
 & \multirow{3}{*}{L-30b} & WikiMIA-32 & 28m 40s & $4.05 \times 10^{-5} \pm 3.10 \times 10^{-6}$ s \\
 & & WikiMIA-64  & 5m 8s & $4.96 \times 10^{-5} \pm 2.54 \times 10^{-5}$ s \\
 & & BookMIA & 16m 38s & $3.94 \times 10^{-5} \pm 1.42 \times 10^{-5}$ s \\
\cmidrule{2-5}
 & \multirow{3}{*}{P-6.9} & WikiMIA-32 & 24m 55s & $2.91 \times 10^{-5} \pm 4.41 \times 10^{-6}$ s \\
 & & WikiMIA-64  & 26m 13s & $3.18 \times 10^{-5} \pm 1.56 \times 10^{-5}$ s \\
 & & BookMIA & 18m 23s & $2.86 \times 10^{-5} \pm 6.16 \times 10^{-6}$ s \\
\cmidrule{2-5}
 & P-12b & BookMIA & 19m 49s & $4.07 \times 10^{-5} \pm 4.89 \times 10^{-6}$ s \\
\cmidrule{2-5}
 & \multirow{2}{*}{M-1.4b} & WikiMIA-32 & 1h 6m 18s & $3.87 \times 10^{-5} \pm 1.27 \times 10^{-6}$ s \\
 & & WikiMIA-64 & 1h 16m 51s & $3.12 \times 10^{-5} \pm 1.42 \times 10^{-5}$ s \\
\bottomrule
\end{tabular}
}
\label{tab:LOS-Net full timing}
\end{table*}

\section{Additional Tasks Background}
\label{app:additional background}
In this section, we provide some additional background and motivation for the DCD and HD tasks.
\xhdr{Data Contamination Detection.} 
Large-scale pre-training of LLMs typically involves crawling vast amounts of online data, a common practice to meet their substantial data requirements. However, this approach risks exposing models to evaluation datasets, potentially compromising our ability to assess their generalization performance accurately \cite{NEURIPS2020_1457c0d6}, or, taking a different perspective, can pose legal and ethical issues when models are accidentally exposed to copyrighted or sensitive data during training. This phenomenon is typically referred to as Data Contamination. Recently, \citet{li-etal-2024-open-source} demonstrated that LLMs from the widely used LLaMA \cite{touvron2023llama} and Mistral \cite{jiang2023mistral} model families exhibit significant data contamination, particularly concerning frequently used evaluation datasets.

\xhdr{Hallucination Detection.}
LLMs' tendency to generate untrustworthy outputs, commonly known as "hallucinations," remains a significant challenge to their widespread adoption in real-world applications \cite{tonmoy2024comprehensive}. To address this issue, various hallucination mitigation techniques have been proposed, including retrieval-augmented generation \cite{lewis2020retrieval,izacard2023atlas,gao2023retrieval}, customized fine-tuning \cite{maynez-etal-2020-faithfulness,cao-etal-2022-learning,qiu-etal-2023-detecting}, and, inference-time manipulation \cite{li2024inference,qiu2024spectral,zhao-etal-2024-enhancing}, to name a few. However, applying these methods to all user-LLM interactions can be computationally expensive. As a more targeted approach, hallucination detection has been explored to enable selective intervention only when necessary.

\xhdr{General Considerations on Annotations.} We consider access to a set of annotations $y$'s, which we naturally associate with the corresponding LOS elements.
These encode ground-truth labels pertaining to problems of interest, e.g., whether the input text $\vec{s}$ is in the pretraining corpus of $f$, or whether $f$ hallucinated when generating $\vec{g}$ from prompt $\vec{s}$.
Collecting these annotations is generally possible, and various strategies could be adopted. For example, for DCD, labels can be gathered with collaborative efforts testing for text memorization, as studied e.g.\ in~\citep{chang2023speak}. We also note that annotations are immediately (and trivially) available for open-source LLMs with disclosed pretraining corpora such as Pythia~\citep{biderman2023pythia}. 
As we demonstrated in \Cref{sec:exp}, models trained on annotations available for one LLM can, in some cases, be \emph{transferred} and applied to another LLM. 

For HD, ground-truth labels can be collected by providing the target LLM with inputs prompting for completion or question answering on known facts and/or reasoning tasks. 
Hallucinations or error annotations are derived by comparing the consistency of the model's response with known, factually true, or logically correct answers. For further details, refer to \Cref{app:Datasets for Hallucination Detection}.




\newpage

\section{LOS-Net Visualization}
\label{app:LOS-Net Visualization}
In \Cref{fig:our-arch} we provide a visualization of our architecture, \ourmethod\ .
\begin{figure}[H]
    \centering
    \includegraphics[width=\columnwidth]{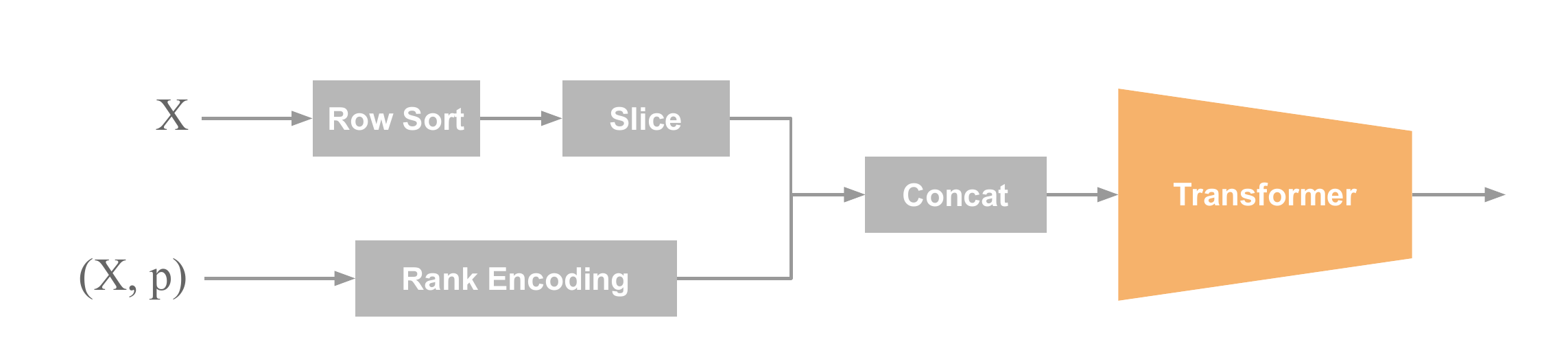}
    \caption{A visualization of \ourmethod\ .}
    \label{fig:our-arch}
\end{figure}

\section{LLM Processing Pipeline}
\label{app:LLM Processing Pipeline}
In \Cref{fig:llm generation} we provide a visualization of the LLM processing pipeline.

\begin{figure}[H]
    \centering
    \includegraphics[width=0.8\columnwidth]{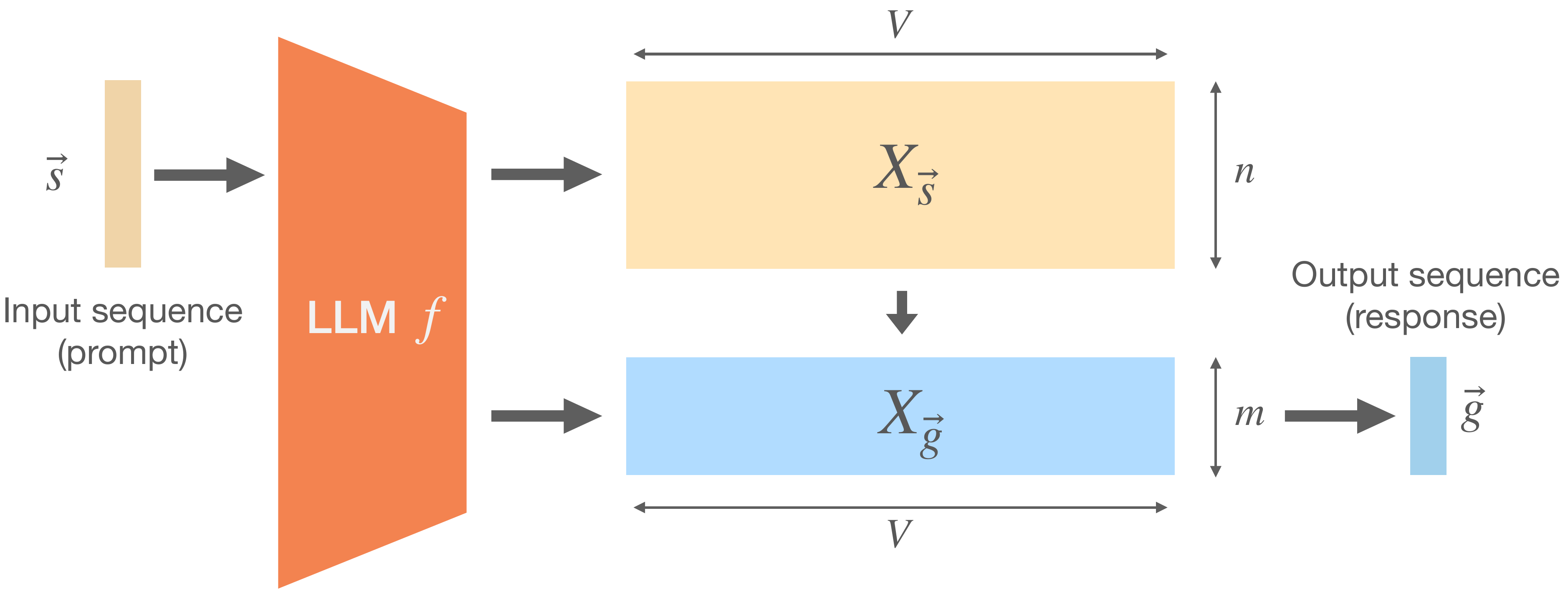}
    \caption{LLM processing pipeline. Token sequence $\vec{s}$ is processed by an LLM $f$, generating full TDSs $\textbf{X}_s, \textbf{X}_g$ for input $\vec{s}$ and response $\vec{g}$.}
    \label{fig:llm generation}
\end{figure}

\section{Importance of TDS Illustration}
\label{app:Importance of TDS Illustration}
We demonstrate the importance of the TDS tensor through the following example, see \Cref{fig:TDS imp}

\begin{figure}[H]
    \centering
    \includegraphics[width=0.8\linewidth]{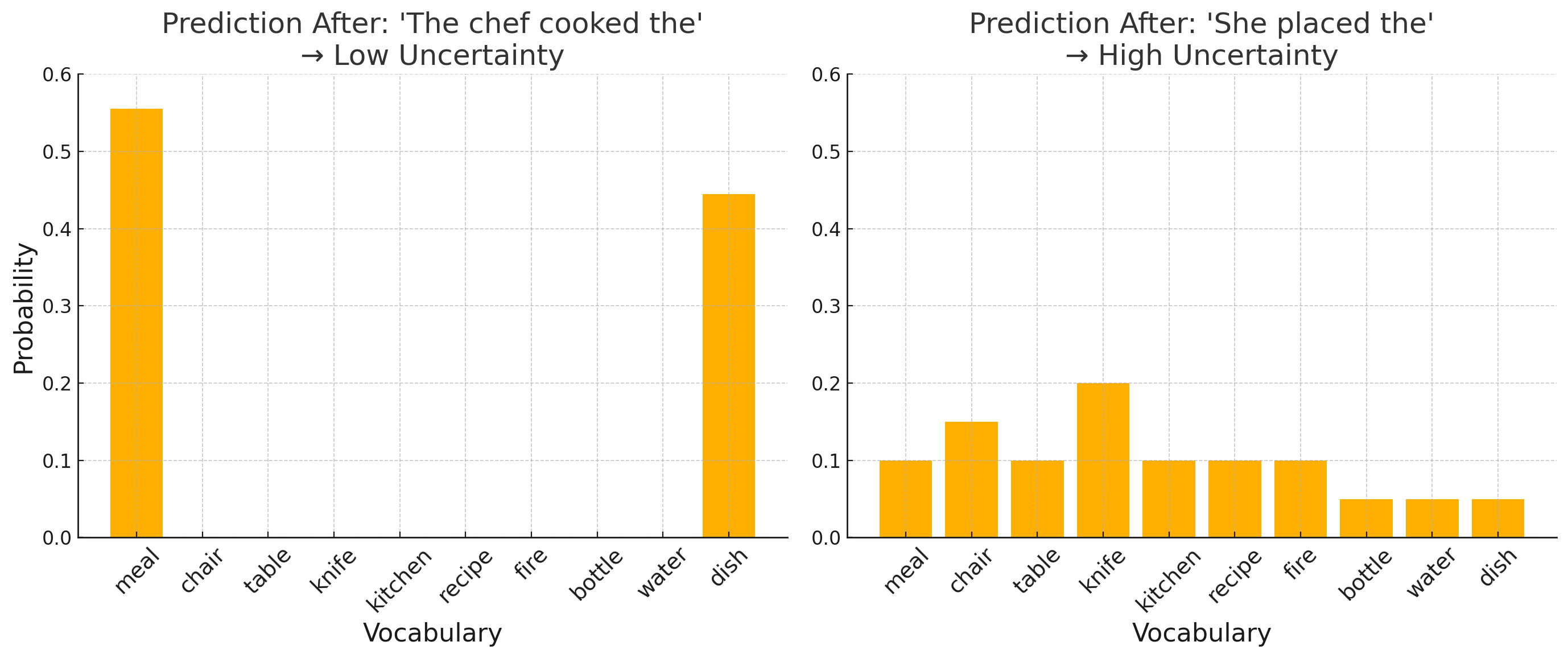}
    \caption{Illustrative example of the importance of the TDS.}
    \label{fig:TDS imp}
\end{figure}

\end{document}